%% file: main.tex
\documentclass[11pt]{article}
\pdfoutput=1
\input{mysty}
\usepackage {color}

\begin{document}

\title{\huge A Continuous-Time Reinforcement Learning Framework for Fine-Tuning Discrete Diffusion Models} 
\author{
Zikun Zhang\thanks{Department of Industrial Engineering and Operations Research, Columbia University, New York, NY 10027.
Emails: \texttt{zz3367@columbia.edu}, \texttt{js6646@columbia.edu},
\texttt{yao@columbia.edu}, \texttt{wt2319@columbia.edu}.}
\and
Jiayuan Sheng\footnotemark[1]
\and
David D. Yao\footnotemark[1]
\and
Wenpin Tang\footnotemark[1]
}
\date{}
\maketitle

\begin{abstract}
We formulate reinforcement learning (RL) in continuous time with discrete state spaces and
possibly arbitrary action spaces via a stochastic control approach,
where the state dynamics are modeled as a controlled continuous-time Markov chain (CTMC).
We consider policy optimization problems
and derive corresponding policy gradient methods, 
leading to continuous-time variants of proximal policy optimization (PPO) and group relative policy optimization (GRPO). 
As a primary application, we develop a continuous-time RL framework for fine-tuning score-based discrete diffusion models, which enables reward-driven optimization without requiring differentiability on the reward signals.
In contrast to the existing GRPO-based approaches that only rely on terminal rewards, our formulation allows intermediate reward or advantage signals to be incorporated throughout the denoising trajectory. 
Importantly, when specialized to masked diffusion models (MDMs), our framework encompasses a rich class of policy parameterizations over the vocabulary simplex with analytically tractable probability ratios, 
providing a unified perspective on exploration and policy optimization in MDMs. For diffusion large language models (dLLMs), we further propose trajectory subsampling techniques to efficiently estimate computationally prohibitive trajectory likelihoods, reducing the computational cost of computing per-position probability ratios. We showcase the effectiveness of our methods on both low-dimensional entropy-regularized optimization problems and RL post-training of dLLMs on reasoning and coding tasks.


\end{abstract}

\input{section1}

\input{section2}

\input{section3}

\input{section4}

\input{section5}

\input{section6}

\section*{Acknowledgement}
Tang is supported by NSF CAREER Award DMS-2538791 and the Tang Family Assistant Professorship.
Sheng and Zhang are supported by NSF Grant DMS-2206038.
This research is part of a Columbia-CityU/HK collaborative project that is supported by InnoHK
Initiative, The Government of the HKSAR and the AIFT Lab.

{
        \bibliographystyle{plainnat}
        \bibliography{ref.bib}
}

\input{appendix}

\end{document}

%% file: mysty.tex
\usepackage[colorlinks=true, linkcolor=red, citecolor=blue, urlcolor=blue]{hyperref}

\usepackage[round]{natbib}

\usepackage{algorithm}
\usepackage{algorithmicx}
\usepackage{algpseudocode}

\usepackage[utf8]{inputenc} 
\usepackage[T1]{fontenc}    
\usepackage[english]{babel}

\usepackage{amsfonts}
\usepackage{nicefrac}       
\usepackage{microtype}      
\usepackage{lmodern}
\usepackage{amssymb,amsmath,amsthm}
\usepackage{stmaryrd}
\usepackage{bbm,bm}
\usepackage{latexsym}
\usepackage{xcolor}

\usepackage{enumerate}
\usepackage{enumitem}
\usepackage{verbatim}
\usepackage{booktabs}       
\usepackage{multirow}
\usepackage[most]{tcolorbox}
\usepackage{caption}
\usepackage{subcaption}

\usepackage{url}            
\usepackage[colorlinks=true]{hyperref}
\usepackage[capitalise,noabbrev]{cleveref}
\usepackage{caption}

\usepackage{parskip}
\usepackage{geometry}
\usepackage[short]{optidef}
\usepackage{algorithm}

\usepackage{thmtools}

\declaretheorem{theorem}

\declaretheorem{lemma}
\declaretheorem{proposition}

\declaretheoremstyle[qed=$\square$]{definitionwithend}

\declaretheorem[style=definitionwithend]{example}
\declaretheorem[style=definitionwithend]{remark}

\crefname{assumption}{Assumption}{Assumptions}
\crefname{fact}{Fact}{Facts}
\crefname{question}{Question}{Questions}
\crefname{observation}{Observation}{Observations} 

\definecolor{gold}{rgb}{0.85,0.65,0}

\def\N{{\mathbb{N}}}

\def\P{{\mathbb{P}}}

\def\R{{\mathbb{R}}}

\def\Z{{\mathbb{Z}}}

\def\cA{{\cal A}}

\def\cD{{\cal D}}

\def\cF{{\cal F}}

\def\cL{{\cal L}}
\def\cM{{\cal M}}
\def\cN{{\cal N}}

\def\cS{{\cal S}}
\def\cT{{\cal T}}

\def\cV{{\cal V}}

\def\cX{{\cal X}}

\def\a{{\mathbf a}}

\def\e{{\mathbf e}}
\def\f{{\mathbf f}}

\def\o{{\mathbf o}}
\def\p{{\mathbf p}}

\def\x{{\mathbf x}}
\def\y{{\mathbf y}}

\DeclareMathOperator*{\argmin}{arg\,min}
\DeclareMathOperator*{\argmax}{arg\,max}

\DeclareMathOperator*{\E}{\mathbb{E}}

\newcommand{\rmd}{\mathrm{d}}

\usepackage{soul}

\usepackage{amsmath,amssymb,amsthm}

\usepackage{amsthm}

\newcommand{\inner}[2]{\langle #1,\, #2 \rangle}

%% file: section1.tex
\section{Introduction}

Score-based diffusion models, which was originally designed for image generation in continuous space 
\citep{ho2020denoising, song2021score}, have recently emerged as a powerful paradigm for discrete data generation in the context of LLMs \citep{austin2021structured,campbell2022continuous,sun2023scorebased,lou2023discrete}. 
For discrete diffusion models with the forward process formulated by a continuous-time Markov chain (CTMC), one learns the \textit{discrete score function} \citep{lou2023discrete,meng2022concrete,benton2024denoising}, 
i.e., the probability ratio of the forward marginals, to generate the reverse process. 
A notable example is the masked diffusion model \citep{shi2024simplified,sahoo2024simple} that corrupts data into fully masked sequences and trains a denoiser to predict the clean data from the masked tokens.
As was shown by \cite{shi2024simplified}, the score function reduces to the denoiser probability vector on the vocabulary space for masked positions, making the model formulation, training, and inference much simpler by directly predicting the clean data at each denoising step ($x_0$-prediction).
This simplified structure gains dominant popularity, and further enhances the rapid development of mask-based diffusion large language models \citep{nie2026large,ye2025dream,khanna2025mercury,liu2025wedlm,cheng2025sdar,zhu2026llada,nie2026improved}.
dLLMs mark a promising generative modeling paradigm for text generation because they enable more flexible parallel and any-order token decoding, in contrast to the traditional left-to-right sequence decoding in autoregressive large language
models (ARMs). 

Fine-tuning discrete diffusion models via reinforcement learning (RL) is challenging because sampling from the discrete categorical distribution is non-differentiable, so standard gradient methods are not applicable. 
Specific to language models, the ARM fine-tuning process is naturally constructed as a Markov Decision Process (MDP) thanks to the left-to-right decoding, in which the action is to predict the next token with the state being the current decoded sequence, leading to a factorized policy likelihood by the chain rule. 
However, it is difficult to scale and enhance the reasoning capabilities of dLLMs by online RL that has been well-developed in ARM post-training,
because of the intractable likelihood of the generated sequences in any-order decoding.
Recent methods, e.g., d1 \citep{zhao2025d}, UniGRPO \citep{yang2026mmada}, wd1 \citep{tang2026wd}, SPG \citep{wang2026spg}, and d2 \citep{wang2026d2}, focused on estimating the trajectory log-likelihood, either by heuristic estimators lacking a clear probabilistic interpretation or by tractable evidence lower and upper bounds. While these approaches have achieved promising empirical performance, they often suffer from coarse likelihood approximations and limited control over the estimation errors. 
More fundamentally, existing GRPO methods are not grounded in a unified MDP formulation, leaving the connection between iterative diffusion denoising for rollout generation and RL policy optimization largely heuristic and inconsistent. Moreover, most GRPO-based algorithms rely on reward signals only at the final completion, resulting in coarse credit assignment for the intermediate token-unmask decisions that shape the generation process. 
Although a recent work \citep{oba2026diffusion} introduced branching-based resampling at selected intermediate steps to provide intermediate feedback, 
the rewards are still being derived from the terminal evaluation,
and are potentially computation-expensive due to branching.

The purpose of this work is to tackle the aforementioned issues
by developing an RL framework for CTMC models,
which enables us to propose a principled approach for fine-tuning \textit{any} score-based discrete diffusion model via continuous-time reinforcement learning (CTRL).
To see the motivation, 
\cite{zhao2024scores, zhao2025score} and \cite{gao2025reward} first proposed to fine-tune diffusion models in continuous spaces via CTRL.
In that setting, 
the score function is the only unknown component in the generative process,
so one can treat it as a control/action/policy,
and hence, the reverse process is a state-controlled stochastic differential equation (SDE).
Then a score-related stochastic policy is designed for exploration,
and the corresponding policy gradient and estimation can be obtained by the 
well-developed stochastic control and RL theory in continuous time and space 
\citep{wang2020reinforcement,jia2022policy,jia2022policyevaluation,jia2023q,zhao2023policy,zhao2026policy,tang2024regret}.
Since the state dynamics of diffusion models are inherently formulated as continuous-time processes (either SDEs or CTMCs),
fine-tuning them via CRTL is a natural choice,
which not only exhibits robustness to time discretization \citep{zhao2025score},
but also allows us to carry out all clean theoretical derivations in continuous time,
while retain great flexibility for discretizing the continuous-time processes at the final algorithmic stage.

Controlled CTMCs on discrete state spaces have a long history in operations research and applied probability under the name of continuous-time Markov decision processes (CTMDPs) \citep{howard1960dynamic, miller1968finite, guo2009continuous, piunovskiy2020continuous}, where the emphasis is on optimality equations and on the existence and computation of optimal policies when the transition rates and rewards are known. However, a model-free RL theory for CTMDPs in the spirit of \citet{wang2020reinforcement} and \citet{jia2023q} remains largely undeveloped. It is also unclear whether control-based RL fine-tuning carries over to discrete diffusion models.
In this work, we address these problems by formulating 
a CTMC framework for CTRL with discrete state spaces and arbitrary (continuous or discrete) action spaces.  
On one hand, this framework opens a new gate for research at the intersection of control theory and RL in discrete spaces, where policy evaluation, policy gradient theory, and practical algorithms remain underexplored. 
On the other hand, we demonstrate its application to fine-tuning discrete diffusion models by treating the denoiser probability vector as the action, 
which offers flexibility for policy parameterization design.
This perspective naturally induces a controlled CTMC that explores the output distribution of the denoiser,
and yields a clean analytical expression for the policy log-likelihood. Consequently, it bridges the gap between autoregressive-style policy optimization objectives and diffusion-based generative modeling, providing a principled RL framework for fine-tuning discrete diffusion models. 

\subsection{Contributions}
\begin{enumerate}[leftmargin=*]
\item
{\em RL theory and algorithms}.
 We cast RL in continuous time and discrete spaces as a controlled CTMC and develop a model-free learning framework on top of it, where the controlled CTMC dynamics recover classical CTMDP results.
We build foundations of CTRL in discrete space by deriving the corresponding instantaneous advantage rate ($q$-function) and policy gradient. By utilizing these, we propose PPO and GRPO algorithms for general CTRL in discrete spaces.

\item 
{\em Fine-tuning discrete diffusion models}.
Inspired by the works \citep{zhao2025score,gao2025reward}, 
we view the sampling process of discrete diffusion models as a stochastic control problem, 
where the discrete score function is viewed as the control/action/policy. 
So the problem of fine-tuning discrete diffusion models naturally fits into the CTRL in discrete spaces, thereby solving it via the aforementioned RL algorithms.
To this end,
we propose several policy parameterization methods (Dirichlet policy, temperature softmax policy, and logistic normal policy), 
and derive the corresponding policy probability ratios. 
Our framework is flexible with any predefined process reward models that assign credits to intermediate partially denoised states. 
We also adopt trajectory subsampling to reduce the number of forward passes during training.

\item 
{\em Experiments}.
We conduct experiments on both low-dimensional synthetic data \citep{so2026discrete},
and LLaDA \citep{nie2026large}, an open-sourced 8B dLLM, for reasoning and coding tasks.
For the low-dimensional synthetic example, 
we illustrate that the PPO algorithm outperforms its alternatives (d1 \citep{zhao2025d} and DAM \citep{so2026discrete})
in both generation performance and convergence.
For LLaDA experiments, 
we train our CTRL on the GRPO algorithms,
demonstrating strong effectiveness in both mathematical reasoning tasks (Sudoku, GSM8K and MATH500),
and coding tasks (HumanEval and MBPP)
compared to the state of the arts algorithms -- d1 \citep{zhao2025d}, d2 \citep{wang2026d2} and SPG \citep{wang2026spg}.
For instance, we achieve $88.2\%$ accuracy
on Sudoku with 0-shot prompting training.

\end{enumerate}

\subsection{Related Works}

\textbf{Continuous-Time Markov Decision Processes.}
The controlled CTMC in Section~\ref{sec2} is an instance of a continuous-time Markov decision process (CTMDP), a classical model dating back to \citet{howard1960dynamic}. Early work established optimality equations and policy/value iteration for finite state and action spaces under finite-horizon and discounted criteria \citep{miller1968finite, kakumanu1971continuously}. The uniformization technique \citep{lippman1975applying} showed that CTMDPs with bounded, time-homogeneous transition rates can be reduced to equivalent discrete-time MDPs \citep{puterman2014markov}. The theory has since been extended to countable and Borel state spaces, unbounded transition rates, and general (including relaxed and randomized) policies \citep{guo2009continuous, guo2015finite, piunovskiy2020continuous}. 

On the learning side, \citet{bradtke1994reinforcement} extended temporal-difference learning and Q-learning to semi-Markov decision problems under discounted infinite-horizon criteria. In concurrent work, \citet{meng2026reinforcement} studied RL for continuous-time jump Markov decision processes. Similar to our work, they formulate the HJB equation and define the exploratory transition rate and $q$-function, but they additionally incorporate jump rewards, develop q-learning, and focus on network dynamic pricing. Our work differs from this literature in four respects. First, we consider a finite horizon with time-inhomogeneous rates and a terminal reward, which matches the structure of diffusion sampling. Second, we adopt the exploratory, model-free viewpoint of \citet{wang2020reinforcement} and derive a $q$-function and a policy gradient theorem for parameterized stochastic policies over general, possibly continuous and high-dimensional action spaces, such as the product of simplices used for MDMs. Third, we develop PPO- and GRPO-type algorithms rather than dynamic-programming or Q-learning schemes. Fourth, we deliberately retain the continuous-time formulation instead of uniformizing, so that the theory is discretization-free and the time grid is chosen only at the algorithmic stage. Notably, the reverse rates of MDMs scale as $1/(1-t)$ and are unbounded near the terminal time, which connects our setting to the unbounded-rate CTMDP literature \citep{guo2009continuous, guo2015finite}.

\textbf{Continuous-Time RL for Fine-Tuning Diffusion Models.}
\cite{zhao2024scores,zhao2025score} introduced a continuous-time RL pipeline to fine-tune continuous diffusion models, where the action is parameterized by Gaussian with mean as the score network, and proposed an actor-critic type PPO algorithm by leveraging $q$-function. \cite{gao2025reward} used a similar idea but enabled pretraining a diffusion model without knowing any prior of score function or learning the score. They theoretically proved that the optimal control is Gaussian distribution, and developed an actor–critic type $q$-learning algorithm to solve
the continuous-time RL problem.

\textbf{Fine-Tuning Discrete Diffusion Models via RL.} In the early stage, policy gradient methods for fine-tuning score-based discrete diffusion models were proposed to tackle the non-differentiability of categorical distribution inherent in discrete diffusion model sampling or non-differentiability of reward models when maximizing the entropy-regularized reward objective. DRAKES \citep{wang2025fine} made the originally non-differentiable trajectories differentiable using the Gumbel-Softmax trick for the general CTMC-based discrete diffusion models, but the reward function is required to be differentiable. Similar to our work, they also derived the optimal value function and HJB equation specified to CTMC-based diffusion model dynamics, 
while our derivation is for the more general controlled CTMC diffusion processes. Score Entropy Policy Optimization \citep{zekri2026fine} fine-tuned general discrete diffusion models over non-differentiable rewards by deriving importance sampling gradient. Later, the rising of mask-based dLLMs \citep{nie2026large,ye2025dream} drived a line of work focusing on enhancing their reasoning abilities via RL post-training. Due to fundamental modeling differences between the auto-regressive and diffusion generative paradigms for LLMs, dLLMs enable parallel or any-order decoding at the cost of intractable or computationally prohibitive trajectory likelihood. D1 \citep{zhao2025d} was the first to scale the reasoning ability of dLLMs with GRPO by using mean-field approximation to estimate the intractable policy log probability in a single forward pass. SPG \citep{wang2026spg} utilized the advantage-weighted GRPO framework and estimated the log likelihood by sandwiching its evidence lower bound (ELBO) and deriving an evidence upper bound (EUBO), where the evidence bounds are estimated via Monte Carlo sampling so multiple forward passes are needed. 
WD1 \citep{tang2026wd} adopted the probability ratio estimation from d1 but integrated the estimation of the old and reference policy log likelihoods into an exponential-weighted objective. UniGRPO \citep{yang2026mmada} estimated the log likelihoods via either a one-step unmasking from d1 or Monte Carlo estimation using the ELBO. DIFFPO \citep{zhao2025diffpo} proposed a two-times mean-field approximation by conditioning on one additional latent at a randomly sampled timestep at each optimization step, yielding a better surrogate policy and a loss contained two importance sampling probability ratios. D2 \citep{wang2026d2} approximated trajectory likelihood by products of tractable factors over blocks, reducing the number of forward passes from the number of denoising steps to the number of blocks. ESPO \citep{ou2026principled} modeled entire-sequence generation as a single action and uses the ELBO as a tractable approximation to the sequence-level likelihood. LFPO \citep{wei2026lfpo} overcame the likelihood intractability by directly optimizing denoising logits via contrastive positive/negative trajectories with significantly faster inference. In order to exploit the information of intermediate sequence, DISPO \citep{oba2026diffusion} proposed to branch from an intermediate sequence by resampling the currently masked positions from rollout-cached logits, scores the resulting completions by the reward model, and updates only the newly filled tokens. TraceRL \citep{wang2026revolutionizing} introduced a diffusion-based value model
that enhances training stability and accommodates the process rewards, providing stronger reward supervision. DTRPO \citep{zhang2026dtrpo} adapted direct policy optimization (DPO) \citep{rafailov2023direct} to dLLM policy optimization and proposed to estimate the trajectory log likelihood by leveraging subsampling denoising time steps and block-attention mask. Although this estimator can be implemented in a single forward pass, dTRPO is an offline RL method via constructed preference set; estimating the trajectory likelihood in online RL methods like GRPO in a single forward pass is still challenging. LLaDOU \citep{huang2026reinforcing} defined action at each step to be first determine the set of tokens to unmask and second predict the values of these tokens to obtain the next-time sequence, where some score is predicted to rank masked tokens under the current state, which can be viewed as a specific action policy in our framework.
Although JustGRPO \citep{ni2026the} indicated that any-order generation during RL training may limit the reasoning capabilities of dLLMs, we still focus on maintaining consistency between rollout generation and policy optimization to align with the continuous-time RL dynamics, as our fine-tuning framework is designed to be general rather than specific to dLLMs.

\medskip
\textbf{Organization of the paper}: The remainder of the paper is organized as follows. In Section \ref{sec2}, we formulate a general RL framework in continuous time and discrete spaces.
Section \ref{sec3} focuses on the application to fine-tuning discrete diffusion models, 
with the special case of masked diffusion models discussed in Section \ref{sec4}.
We report the experimental results in Section \ref{sec5}.
Concluding remarks are provided in Section \ref{sec6}.

%% file: section2.tex
\section{An RL Framework in Continuous Time and Discrete Spaces}
\label{sec2}
\subsection{Model Formulation} 
We formulate CTRL for CTMCs with discrete state spaces.
This complements the work of CTRL for SDEs \citep{wang2020reinforcement,jia2022policy,jia2022policyevaluation,jia2023q,zhao2023policy}
and for jump diffusions \citep{gao2024reinforcement},
both with continuous state spaces.

Let $\cS$ be a discrete state space, and $\cA$ be an action space. 
With an abuse of notation, for continuous-time dynamics (and theoretical development), $t\in [0,T]$ where $T\in\R_+$ is a finite time horizon; 
while in the final algorithmic stage where the continuous time range is discretized, $t\in [T]$ denotes the time step with $T\in \N_+$ being the number of time discretization steps. 

We first introduce a controlled transition-rate map $R: [0,T]\times \cS \times \cA \to \R^{|\cS|}$, such that for each fixed $(t,x,a)\in [0,T]\times\cS\times\cA$, it holds that
\begin{equation}\label{rate}
R(t,x,a)_y\geq 0 \quad \text{for}\ y \neq x,\  y\in \cS, \quad \text{and}\quad R(t,x,a)_x=-\sum_{y\neq x} R(t,x,a)_y.
\end{equation}

If we ignore the control $a$ in \eqref{rate}, 
$R$ is the rate matrix of a time-inhomogeneous CTMC, with the $(x,y)$-th entry given by $R(t,x)_y$ at a fixed time $t\in[0,T]$ (see e.g., \citet[Chapter 2]{Liggett}). 
Thus, \eqref{rate} can be understood as an action-controlled rate matrix for any fixed $t$. For the realization of such an $R$, we give an example that is linear in $a$, encompassing a rich family of transition-rate maps.

\smallskip
\begin{example}[Rate with Linear Action]\label{example:rate}
Let the action space be $\cA \subset \R_+^d$. For any matrix-valued map $\Phi: [0,T]\times \cS \to\R^{|\cS|\times d}$ such that for any $(t,x)\in [0,T]\times \cS$, all columns of $\Phi(t,x)$ sum to zero: $$\sum_{y\in\cS}\Phi(t,x)_{y,i} = 0,\quad \forall i\in[d],$$
$R(t,x,a)=\Phi(t,x)\cdot a$ is a valid rate map if the off-diagonal rates are nonnegative on $\cA$: $R(t,x,a)_y\geq 0$ for all $y\neq x$ and $a\in\cA$, because it holds that
\begin{equation*}
\sum_{y\in\cS} R(t,x,a)_y
= \sum_{y\in\cS}\sum_{i=1}^{d}\Phi(t,x)_{y,i}\,a_i
= \sum_{i=1}^{d}a_i\sum_{y\in\cS}\Phi(t,x)_{y,i}
=0.
\end{equation*}
As will be clear in Proposition~\ref{prop:diffusion}, 
this linear action structure is satisfied by the discrete diffusion embedding.
\end{example}

Next, given an action sequence $\a=(a_t)_{t\in [0,T]}\in \cA^{[0,T]}$, 
we define the controlled state dynamics $(X_t^\a)_{t\in [0,T]}$ on $\cS$ by a CTMC with infinitesimal transition probability:
\begin{equation}\label{action}
p_{t+\Delta t\mid t}(y\mid x_t^\a)=\delta\{x_t^\a,y\}+R(t,x_t^\a,a_t)_y\cdot\Delta t+o(\Delta t), \quad y \in\cS, \ x_0^\a\sim \rho,
\end{equation}
where $\delta\{\cdot, \cdot\}$ denotes the Kronecker delta, $x_t^\a$ is a realization of $X_t^\a$, $a_t$ stands for the action/control at time $t$, and $\rho$ is the initial state distribution. The zero-sum condition in \eqref{rate} ensures that \eqref{action} is a valid infinitesimal transition probability, and \eqref{action} means that $R(t,x_t^\a,a_t)_y$ is the rate at which the probability mass moves from state $x_t^\a$ to $y$ under the control $a_t$. 

Commonly, we consider stochastic policy in RL by exploration to interact with and learn the
unknown environment through trial and error. 
At each time $t$ with the current state $x_t$, an action $a_t$
is generated or sampled from the distribution $\pi(\cdot\mid t,x_t)$. A policy $\pi(\cdot\mid t,x)$ is a probability distribution over $\cA$ given the current time-state pair $(t,x)$. Formally, $\pi(\cdot\mid\cdot,\cdot)$ is a function $\pi:[0,T]\times \cS\to \Delta(\cA)$.  
Denote by $(X_t^\pi)_{t\in [0,T]}$ the state dynamics governed by a feedback policy $\pi(\cdot\mid \cdot, \cdot)$,
which is a CTMC with infinitesimal transition probability:
\begin{equation}\label{pidynamics}
p_{t+\Delta t\mid t}(y\mid x_t^\pi)=\delta\{x_t^\pi,y\}+R(t,x_t^\pi,a_t^\pi)_y\cdot\Delta t+o(\Delta t), \quad y \in\cS, \ x_0^\pi\sim \rho,
\end{equation}
where $x_t^\pi$ is a realization of $X_t^\pi$ and $a_t^\pi\sim \pi(\cdot\mid t,x_t^\pi)$. 
To evaluate such a policy $\pi$, we consider an \textit{exploratory version} of state dynamics $(\tilde{X}_t^\pi)_{t\in [0,T]}$
in the same spirit as RL in continuous time and space \citep{wang2020reinforcement}.
Define its transition mechanism by
\begin{equation}\label{RLprocess}
p_{t+\Delta t\mid t}(y\mid \tilde{x}_t^\pi)=\delta\{\tilde{x}_t^\pi,y\}+\tilde{R}(t,\tilde{x}_t^\pi;\pi(\cdot\mid t,\tilde{x}_t^\pi))_y\cdot\Delta t+o(\Delta t), \quad y \in\cS, \ \tilde{x}_0^\pi\sim \rho,
\end{equation}
where $\tilde{x}_t^\pi$ is a realization of $\tilde{X}_t^\pi$ and $\tilde{R}(t,x;\pi(\cdot)):=\int_\cA R(t,x,a)\pi(a)\,\rmd a=\E_{a\sim \pi} R(t,x,a)$. Note that the external randomization of action $a$ is removed by taking the expectation, leaving a chain governed by the averaged transition rate
$\tilde{R}$. Importantly, we have the following proposition that allows us to perform theoretical analysis with the action-averaged exploratory state dynamics \eqref{RLprocess}, while observing real trajectories by simulating \eqref{pidynamics}. The proof is deferred to Appendix~\ref{app1-0}.

\smallskip
\begin{proposition}\label{prop:exploratory}
The CTMC $(X_t^\pi)_{t\in [0,T]}$ has the same distribution as the exploratory state dynamics $(\tilde{X}_t^\pi)_{t\in [0,T]}$ under another probability measure. 
\end{proposition}

\medskip
\textbf{Performance Metric.} The goal is to find the optimal feedback policy $\pi^*$ that maximize the expected reward:
\begin{equation}\label{RLobjective}
\pi^*=\argmax_{\pi}\E \left[\int_0^T r(t,X_t^\pi,a_t^\pi)\,\rmd t + h(X_T^\pi) \mid X_0^\pi\sim\rho \right],
\end{equation}
where $r: [0,T]\times \cS\times \cA\to \R$ is the running reward function, and $h:\cS\to\R$ is the terminal reward function. 
Also denote the value function of a policy $\pi$ by
\begin{align*}
V(t,x;\pi)={}&\E \left[\int_t^T r(s,X_s^\pi,a_s^\pi)\,\rmd s + h(X_T^\pi) \mid X_t^\pi=x \right]\\
={}&\E \left[\int_t^T \tilde{r}(s,\tilde{X}_s^\pi;\pi(\cdot\mid s,\tilde{X}_s^\pi))\,\rmd s + h(\tilde{X}_T^\pi) \mid \tilde{X}_t^\pi=x \right],
\end{align*}
where $\tilde{r}(t,x;\pi(\cdot)):=\int_\cA r(t,x,a)\pi(a)\,\rmd a=\E_{a\sim \pi} r(t,x,a)$ is the exploratory reward for any policy $\pi\in\Delta(\cA)$.

\subsection{Feynman-Kac Formula} 
Given $a\in\cA$, let $\cL^a$ be the infinitesimal generator associated with the process $(X_t^a)_{t\in [0,T]}$ defined by \eqref{action} (see e.g., \citet[Example 2]{benton2024denoising}):
$$
\cL^a f(t,x):=\frac{\partial f}{\partial t}(t,x)+\langle R(t,x,a), f(t,\cdot)\rangle,\quad f:[0,T]\times \cS\to \R.
$$
Here $\langle\cdot,\cdot\rangle$ denotes inner product between two vectors,
i.e., $\langle R(t,x,a), f(t,\cdot)\rangle=\sum_{y\in\cS} R(t,x,a)_y f(t,y)$.

With the action-aware infinitesimal generator in place, we now state the Feynman-Kac formula for discrete state spaces.
The proof is deferred to Appendix~\ref{app1-1}.

\smallskip
\begin{lemma}[Feynman-Kac Formula]\label{lemma:FK}
There exists $v$ satisfying
$$\mathbb{E}_{a\sim \pi(\cdot\mid t,x)}(\cL^a v(t,x)+r(t,x,a))=0, \quad v(T,x)=h(x), \forall x\in\cS.$$
Then $v$ is the value function: $v(t,x)=V(t,x;\pi), \forall (t,x)\in [0,T]\times \cS$. 
\end{lemma}

\smallskip
\begin{remark}
For any $(t,x,a)\in [0,T]\times \cS\times \cA$ and $f:[0,T]\times \cS\to \R$, let
$$
H(t,x,a,f):=\langle R(t,x,a), f(t,\cdot)\rangle+r(t,x,a).
$$
Then $H$ is the (generalized) Hamiltonian, an analogue to Hamiltonian in classical stochastic control. The optimal value function is given by
$$
V^*(t,x):=\sup_{\a=(a_s)_{s\in[t,T]}}\E\left[\int_t^T r(s,X_s^\a,a_s)\, \rmd s+h(X_T^\a) \mid X_t^\a=x \right],
$$
where $(t,x)\in [0,T]\times \cS$. 
The associated HJB equation for $V^*$ is:
$$
\frac{\partial v}{\partial t}(t,x)+\sup_{a\in\cA} H(t,x,a,v)=0,\quad v(T,x)=h(x), \ \forall x\in\cS,
$$
and the optimal feedback policy is
$a_t^*=a^*(t,x)=\argmax_{a\in\cA} H(t,x,a,V^*)$,
where $(t,x)$ is the current time-state pair.

In the standard control setting where the model is fully known (i.e., functional forms of $R, r$ and $h$ are specified), the opitmal policy can be derived at $t=T$ and will be carried out through $[0,T]$ by dynamic programming, so the learning and exploration are not needed. 
In the RL setting,
an additional entropy regularizer is usually adopted to encourage exploration (but for our purpose of fine-tuning diffusion models, this term is not needed):
\begin{align*}
V_\gamma(t,x;\pi):={}&\E \left[\int_t^T (r(s,X_s^\pi,a_s^\pi)-\gamma \log \pi(a_s^\pi)) \,\rmd s + h(X_T^\pi) \mid X_t^\pi=x \right].
\end{align*}
The value function $V_\gamma^*(t,x):= \sup_\pi V_\gamma(t,x;\pi)$ satisfies the exploratory HJB equation \citep{tang2022exploratory}:
$$
\frac{\partial v}{\partial t}(t,x)+\sup_{\pi} \mathbb{E}_{a\sim \pi}(H(t,x,a,v)-\gamma \log\pi(a))=0,\quad v(T,x)=h(x), \ \forall x\in\cS,
$$
and the optimal feedback policy is
$\pi^*_\gamma(a\mid t,x)\propto \exp\left(\frac{1}{\gamma}H(t,x,a,V_\gamma^*)\right)$.
\end{remark}

\subsection{Policy Gradient} 
In practice, the policy is parameterized by a function family $\{\pi^\theta(\cdot\mid\cdot.\cdot)\}_\theta$. We write $(X_t^\theta)_{t\in [0,T]}$ for the process $(X_t^{\pi^\theta})_{t\in [0,T]}$ governed by the policy $\pi^\theta$. Its transition probability is:
$$
p^\theta_{t+\Delta t\mid t}(y\mid x^\theta_t)=\delta\{x^\theta_t,y\}+\tilde{R}(t,x^\theta_t;\pi^\theta(\cdot\mid t,x^\theta_t))_y\cdot\Delta t+o(\Delta t), \quad y \in\cS, \ x_0^\theta\sim \rho,
$$
where $x_t^\theta$ is a realization of $X_t^\theta$,
and we denote by $p_t^\theta$ its distribution (a probability mass function). 
Also denote its corresponding value function by $V^\theta(t,x)=V(t,x;\pi^\theta(\cdot\mid t,x))$ and $V^\theta=\E_{x\sim\rho}V^\theta(0,x)$. The following theorem computes $\nabla_\theta V^\theta$, whose proof is given in Appendix~\ref{app1-2}.

\begin{theorem}[Policy Gradient]\label{lemma:PG}
We have
$$
\nabla_\theta V^\theta=\E \left[\int_0^T \nabla_\theta\log \pi^\theta(a_t^\theta\mid t,X_t^\theta)\cdot q(t,X_t^\theta,a_t^\theta;\pi^\theta)\,\mathrm{d}t  \right],
$$
where $q(t,x,a;\pi):=\cL^aV(t,x;\pi)+r(t,x,a)=\frac{\partial V}{\partial t}(t,x;\pi)+\langle R(t,x,a), V(t,\cdot;\pi)\rangle+r(t,x,a)=\frac{\partial V}{\partial t}(t,x;\pi)+H(t,x,a,V(\cdot,\cdot;\pi))$.
\end{theorem}

\smallskip
\begin{remark}
The $q(t,x,a;\pi)$ function defined in Theorem~\ref{lemma:PG} is the instantaneous advantage rate analogue to the continuous space counterpart \citep{jia2023q}. Following the presentation of \cite{jia2023q}, given $\pi$ and $(t,x,a)\in [0,T)\times \cS\times \cA$, 
consider a perturbed policy $\hat{\pi}$ of $\pi$, 
which takes the constant action $a\in\cA$ on $[t,t+\Delta t)$ for $\Delta t>0$ and then follows $\pi$ on $[t+\Delta t,T]$. Specifically, the corresponding state process $X^{\hat{\pi}}$, given $X_t^{\hat{\pi}}=x$, is broken into two pieces: on $[t,t+\Delta t)$, it is $X^a$ following the rate matrix \eqref{rate}, and on $[t+\Delta t,T]$ it is $X^\pi$ following the transition probability \eqref{pidynamics} with initial time-space pair $(t+\Delta t,X_{t+\Delta t}^a)$. The $q$-function measures the rate of the performance difference between the two policies $\pi$ and $\hat{\pi}$ when $t\to 0$, as shown in the following proposition whose proof is provided in Appendix~\ref{app1-3}.

\smallskip
\begin{proposition}\label{prop:qfunction}
Let ($\Delta t$-parametrized) $Q$-function, denoted by $Q_{\Delta t}(t,x,a;\pi)$, be the expected reward of the perturbed policy $\hat{\pi}$:
$$
Q_{\Delta t}(t,x,a;\pi):=\E\left[
\int_t^{t+\Delta t}r(s,X_s^a,a)\,\rmd s+\int_{t+\Delta t}^T r(s,X_s^\pi,a_s^\pi)\,\rmd s+h(X_T^\pi)\mid X_t^{\hat{\pi}}=x
\right].
$$
Then we have:
$$
q(t,x,a;\pi)=\lim_{\Delta t\to 0}\frac{Q_{\Delta t}(t,x,a;\pi)-V(t,x;\pi)}{\Delta t}.
$$
\end{proposition}
By the zero-sum condition in \eqref{rate}, we get
\begin{align*}
q(t,x,a;\pi)=
{}&\frac{\partial V}{\partial t}(t,x;\pi)+\sum_{y\neq x}R(t,x,a)_y (V(t,y;\pi)-V(t,x;\pi))+r(t,x,a).
\end{align*}
This expression has an intuitive interpretation of the advantage rate. 
The first term is independent of the action and captures the temporal change of the value function. The second term is the value difference weighted by the transition rates, reflecting the space change of the value. In particular, if $V(t,y;\pi)-V(t,x;\pi)>0$,
so that transiting from $x$ to $y$ is advantageous, then a larger transition rate $R(t,x,a)_y$ leads to a higher $q$-value. The final term corresponds to the running reward and quantifies the immediate contribution of the time-state-action triple $(t,x,a)$.
\end{remark}

Based on the derived $q$-function, we present a PPO algorithm in Algorithm~\ref{alg:PPO}, where the value and policy networks are optimized iteratively. Note that we subtract the action-independent term $\partial V/\partial t$ when computing $q$-function $q_{b,t}$. We also present a GRPO algorithm that provides process supervision \citep{shao2024deepseekmath,wang2026process} in Algorithm~\ref{alg:GRPO},  eliminating the need for a value network by directly leveraging the intermediate (running) reward $r$ and the terminal reward $h$. 
While both algorithms employ a uniform time discretization by default, the continuous-time formulation allows for arbitrary discretization schemes. Consequently, more flexible time discretizations in the final algorithmic stage can be adopted for efficiency or accuracy.

\begin{algorithm}[!t]
\caption{Proximal Policy Optimization for Continuous-Time Discrete-Space RL}
\label{alg:PPO}
\begin{algorithmic}[1]
\Require  Initial policy parameter $\theta_0$ and value network parameter $\phi_0$, trajectory batch size $B$, number of time discretization steps $T$, step size $m$, clip parameter $\epsilon$.
\State $\theta\leftarrow \theta_0$, $\phi\leftarrow \phi_0$
\Repeat
\State $\theta_\mathrm{old}\leftarrow \theta$, $\phi_\mathrm{old}\leftarrow \phi$
\State Collect $B$ trajectories $\{\tau_b\}_{b=1}^B$, where $\tau_b=\{x_{b,t},a_{b,t},r_{b,t}\}_{t=0}^{T-1}\cup \{t_{b,T},x_{b,T}\}$ for $t\in [T-1]\cup\{0\}$ by running policy $\pi^{\theta_n}$
\State Compute rewards-to-go $\hat{r}_{b,t}=m\sum_{t'=t}^{T-1} r_{b,t'}+r_{b,T}$ for all $b\in [B]$ and $t\in [T-1]\cup\{0\}$
\State Perform multiple value network updates starting from $\phi_\mathrm{old}$:
$$
\phi\leftarrow \argmin_\phi \frac{1}{B}\sum_{b=1}^B \frac{1}{T}\sum_{t=0}^{T-1}|V^\phi(tm,x_{b,t})-\hat{r}_{b,t}|^2
$$

\State Compute $q_{b,t}=\sum_{y\neq x_{b,t}} R(tm,x_{b,t},a_{b,t})_y \cdot(V^{\phi_{n+1}}(tm,y)-V^{\phi_{n+1}}(tm,x_{b,t}))+r(tm,x_{b,t},a_{b,t})$ for all $b\in [B]$ and $t\in [T-1]\cup\{0\}$
\State Perform multiple policy updates starting from $\theta_\mathrm{old}$:
$$
\theta\leftarrow \argmax_{\theta} \frac{1}{B}\sum_{b=1}^B \frac{1}{T}\sum_{t=0}^{T-1} \min(\rho_{b,t}^\theta\cdot q_{b,t}, \text{clip}(\rho_{b,t}^\theta,1-\epsilon,1+\epsilon)\cdot q_{b,t}),
$$
where $\rho_{b,t}^\theta=\frac{\pi^\theta(a_{b,t}\mid tm,x_{b,t})}{\pi^{\theta_\mathrm{old}}(a_{b,t}\mid tm,x_{b,t})}$ 
\Until{Convergence}
\Ensure $\theta$
\Statex {\footnotesize \textbf{Note:}  Here for sampled trajectories, $x_{b,0}\sim \rho$, $a_{b,t}\sim \pi^{\theta_\mathrm{old}}(\cdot\mid tm,x_{b,t})$, $x_{b,t+1}\sim \mathrm{Cat}(\cdot;e_{x_{b,t}}+R(tm,x_{b,t},a_{b,t})\cdot m)$, $r_{b,t}=r(tm,x_{b,t},a_{b,t})$ for all $b,t$, and $r_{b,T}=h(x_{b,T})$.}
\end{algorithmic}
\end{algorithm}

\begin{algorithm}[!t]
\caption{Group Relative Policy Optimization for Continuous-Time Discrete-Space RL}
\label{alg:GRPO}
\begin{algorithmic}[1]
\Require  Initial policy parameter $\theta_0$, trajectory batch size $B$, trajectory group size $G$, number of time discretization steps $T$, step size $m$, clip parameter $\epsilon$.
\State $\theta\leftarrow \theta_0$
\Repeat
\State $\theta_\mathrm{old}\leftarrow \theta$ 
\State Sample $B$ initial states $\{x_b\}_{b=1}^B \sim  \rho$
\State For each $x_b$, collect $G$ trajectories, $\{\tau_{b,g}\}_{g=1}^G$, starting from $x_b$, where $\tau_{b,g}=\{x_{{b,g,t}},a_{{b,g,t}},r_{{b,g,t}}\}_{t=0}^{T-1}\cup\{t_T,x_{b,g,T},r_{b,g,T}\}$ by running policy $\pi^{\theta_\mathrm{old}}$
\State Compute group mean $\mu_b=\frac{1}{GT}\sum_{g=1}^G\sum_{t=0}^{T-1} r_{b,g,t}$ and $\nu_b=\frac{1}{G} \sum_{g=1}^G r_{b,g,T}$, group standard deviation $\sigma_b=\sqrt{\frac{1}{GT}\sum_{g=1}^G\sum_{t=0}^{T-1}  (r_{b,g,t}-\mu_b)^2}$ and $\omega_b=\sqrt{\frac{1}{G}\sum_{g=1}^G (r_{b,g,T}-\nu_b)^2}$, normalized reward $\hat{r}_{b,g,t}=\frac{r_{b,g,t}-\mu_b}{\sigma_b}$ and $\hat{r}_{b,g,T}=\frac{r_{b,g,T}-\nu_b}{\omega_b}$, and advantage $A_{b,g,t}=\frac{1}{T}\sum_{t'=t}^{T-1} \hat{r}_{b,g,t'}+\hat{r}_{b,g,T}$
\State Perform multiple policy updates starting from $\theta_\mathrm{old}$:
$$
\begin{aligned}
\theta\leftarrow \argmax_{\theta} \frac{1}{B}\sum_{b=1}^B \frac{1}{G}\sum_{g=1}^G \frac{1}{T}\sum_{t=0}^{T-1}
\min(\rho_{b,g,t}^\theta A_{b,g,t}, \text{clip}(\rho_{b,g,t}^\theta,1-\epsilon,1+\epsilon) A_{b,g,t}),
\end{aligned}
$$
where $\rho_{b,g,t}^\theta=\frac{\pi^\theta(a_{b,g,t}\mid tm, x_{b,g,t})}{\pi^{\theta_\mathrm{old}}(a_{b,g,t}\mid tm, x_{b,g,t})}$
\Until{Convergence}
\Ensure $\theta$
\Statex {\footnotesize \textbf{Note:} For sampled trajectories, $x_{b,g,0}=x_b$, $a_{b,g,t}\sim \pi^{\theta_\mathrm{old}}(\cdot\mid tm,x_{b,g,t})$, $x_{b,g,t+1}\sim \mathrm{Cat}(\cdot;e_{x_{b,g,t}}+R(tm,x_{b,g,t},a_{b,g,t})\cdot m)$, $r_{b,g,t}=r(tm,x_{b,g,t},a_{b,g,t})$, and $r_{b,g,T}=h(x_{b,g,T})$} for all $b\in [B],g\in [G],t\in [T-1]\cup\{0\}$.
\end{algorithmic}
\end{algorithm}   

%% file: section3.tex
\section{A General Framework for Fine-Tuning Score-Based Discrete Diffusion Models}\label{sec3}
\subsection{Discrete Diffusion Models} 
In score-based discrete diffusion models \citep{lou2023discrete,sun2023scorebased,campbell2022continuous}, the forward and reverse processes are formulated as CTMCs. Let $Q_t\in\R^{|\cS|\times |\cS|}$ be the transition rate matrix of the forward process $(Y_t)_{t\in [0,T]}$ at time $t\in [0,T]$ such that $\mathrm{Law}(Y_0)=p_\mathrm{data}$ and $\mathrm{Law}(Y_T)\approx p_\mathrm{noise}$.
The reverse process $(X_t)_{t\in [0,T]}$ is the exact time reversal of $(Y_t)_{t\in [0,T]}$ \citep{kelly2011reversibility}, that is, $X_t \overset{\text{d}}{=} Y_{T-t}$, with the rate matrix $Q_t^\leftarrow$ given by
$$
Q_t^\leftarrow (x,y)=Q_{T-t}(y,x)\cdot s(T-t, x)_y\quad \text{for}\ y\neq x,\ y\in \cS,\ \text{and}\ Q_t^\leftarrow(x,x)=-\sum_{y\neq x}Q_t^\leftarrow(x,y),\ \forall x\in\cS.
$$
Here $s(t,x):=(q_t(y)/q_t(x))_{y\neq x}\in \R_+^{|\cS|-1}$ is called \textit{discrete score function} \citep{lou2023discrete,meng2022concrete}, which is a collection of probability ratios of forward marginal distributions. Here, $q_t$ is the forward marginal at time $t$, i.e., the distribution (a probability mass function) of $Y_t$. 

In practice, the score $s(t,x)$ is parameterized by a function family $\{s_\theta(t,x)\}_\theta$,
and is then learned by score matching methods, such as denoising score entropy \citep{lou2023discrete,benton2024denoising}. Once we have a pretrained estimator $s_{\theta_{\mathrm{pre}}}(\cdot,\cdot)$, we can perform generative modeling by discretizing the continuous-time reverse process:
\begin{equation}\label{sampling}
p^{\theta_\mathrm{pre}}_{t+\Delta t\mid t}(y\mid x_t)=\delta\{x_t,y\}+Q_{T-t}(y,x_t)s_{\theta_\mathrm{pre}}(T-t,x_t)_y\cdot  \Delta t+o(\Delta t),\quad y\neq x_t,\ y\in\cS,\ x_0\sim p_\mathrm{noise},
\end{equation}
where $x_t$ is a realization of $X_t$ and  $p_t^{\theta_\mathrm{pre}}$ is the marginal distribution of the pretraining sampling process with score $s_{\theta_{\mathrm{pre}}}(\cdot,\cdot)$ at time $t$.

\subsection{Score as Action} 
To draw connections between CTRL and discrete diffusion models, 
we compare the RL process~\eqref{RLprocess} with the diffusion sampling process \eqref{sampling}. Given the forward diffusion rate matrix $(Q_t)_{t\in [0,T]}$, we choose the transition-rate map $R$ in the RL dynamics \eqref{RLprocess} as
\begin{equation}\label{diffusionrate}
R(t,x,a)_y=Q_{T-t}(y,x)\cdot a_y\quad \text{for}\ y\neq x,\ y\in\cS,
\end{equation}
for any $(t,x,a)\in[0,T]\times \cS\times \cA$. 
Here the action space is specified as $\cA=\R_+^{|\cS|-1}$, 
because we are only concerned with the off-diagnal entries. The following proposition shows that the rate map \eqref{diffusionrate} belongs to the family of rate maps in Example~\ref{example:rate}. 
We give the proof in Appendix~\ref{app:prop3}.

\smallskip
\begin{proposition}\label{prop:diffusion}
Given any $(t,x)\in [0,T]\times \cS$, index the action by the off-diagonal
targets, $a=(a_y)_{y\neq x}\in\R_+^{|\cS|-1}$ (so $d=|\cS|-1$ in Example~\ref{example:rate}), and set
\[
\Phi(t,x)_{y,i} :=
\begin{cases}
Q_{T-t}(i,x), & y = i ,\\[2pt]
-\,Q_{T-t}(i,x), & y = x,\\[2pt]
0, & \text{otherwise},
\end{cases}
\qquad \text{for all}\ y\in \cS \ \text{and}\  i \in\cS: i\neq x.
\]
Then each column of $\Phi(t,x)$ sums to zero, and the rate map \eqref{diffusionrate} satisfies $R(t,x,a)=\Phi(t,x)\cdot a$ for all $(t,x,a)\in [0,T]\times \cS\times \cA$.
\end{proposition}

Suppose that we have access to a pretrained score function $s_{\theta_\mathrm{pre}}(\cdot,\cdot)$, and hence a pretrained model. 
Let $\pi^{\theta_\mathrm{pre}}(\cdot\mid t,x)$ be a deterministic policy $\delta_{s_{\theta_\mathrm{pre}}(T-t,x)}$; that is, if $a_t^{\theta_\mathrm{pre}}\sim \pi^{\theta_\mathrm{pre}}(\cdot\mid t,x_t)$, then $a_t^{\theta_\mathrm{pre}}=s_{\theta_\mathrm{pre}}(T-t,x_t)$. 
So the reverse CTMC \eqref{sampling} becomes
$$
p^{\theta_\mathrm{pre}}_{t+\Delta t\mid t}(y\mid x_t)=\delta\{x_t,y\}+R(t,x_t,a_t^{\theta_\mathrm{pre}})_y\cdot \Delta t+o(\Delta t)\quad \text{for} \ y\in\cS.
$$
As a result, analogous to \citet{zhao2025score}, the score function is replaced by the action, and the problem of finding the optimal score becomes finding the optimal action/policy, which can be tackled by policy optimization. 

\subsection{Post-Training} Generally, for a parameterized policy $\pi^\theta(\cdot\mid \cdot,\cdot)$, 
we denote the control $a_t^\theta\sim \pi^\theta(\cdot\mid t,X_t^\theta)$, where $X_t^\theta$ is short for the reverse process $X_t^{\pi^\theta}$ driven by $\pi^\theta$:
$$
p^{\theta}_{t+\Delta t\mid t}(y\mid x^\theta_t)=\delta\{x^\theta_t,y\}+R(t,x^\theta_t,a_t^{\theta})_y\cdot \Delta t+o(\Delta t)\quad \text{for} \ y\in\cS, \ x_0^\theta\sim p_\mathrm{noise},
$$
where $x_t^\theta$ is a realization of $X_t^\theta$ and $p_t^\theta$ is the distribution of $X_t^\theta$. 
The post-training objective is:
\begin{equation}\label{ptobjective}
\max_\theta \E\left[\mathrm{TRF}(X_T^\theta)+\int_0^T \mathrm{IRF}(X_t^\theta)\,\mathrm{d}t-\beta D_\mathrm{KL}(\P^\theta\,\Vert\, \P^{\theta_\mathrm{pre}})  \right].
\end{equation}
Here $\mathrm{TRF}(\cdot)$ denotes the terminal reward function for the final output $X_T^\theta$, $\mathrm{IRF}(\cdot)$ denotes the intermediate reward function (possibly time-dependent) for the intermediate state $X_t^\theta$, $\P^\theta$ and $\P^{\theta_\mathrm{pre}}$ are the path measures of the processes $(X_t^\theta)_{t\in [0,T]}$ and $(X_t^{\theta_\mathrm{pre}})_{t\in [0,T]}$, respectively, and $\beta>0$ is a penalty hyperparameter. 
The following theorem connects the post-training objective \eqref{ptobjective} with the RL objective \eqref{RLobjective}.
The proof will be given in Appendix~\ref{app1-4}.

\smallskip
\begin{theorem}\label{KLdivergence}
The KL divergence between $\P^\theta$ and $\P^{\theta_\mathrm{pre}}$ can be expressed as:
\begin{align*}
 D_\mathrm{KL}(\P^\theta\,\Vert\, \P^{\theta_\mathrm{pre}})={} &\E_{X_t^\theta\sim p_t^\theta} \int_0^T \sum_{y\neq X_t^\theta} Q_{T-t}(y,X_t^\theta)D_I\left(\E_{a_t^\theta\sim \pi^\theta(\cdot\mid t,X_t^\theta)}(a_t^\theta)_y \,\Vert\, s_{\theta_\mathrm{pre}}(T-t,X_t^\theta)_y\right)\,\rmd t,
\end{align*}
where $D_I(\cdot\,\Vert\,\cdot)$ is the generalized I-divergence \citep{amari2012differential} given by $D_I(x\,\Vert\,y)=-x+y+x\log(x/y)$.
\end{theorem}

By Theorem~\ref{KLdivergence}, the objective in \eqref{ptobjective} is
\begin{equation}\label{ELBO}
\E_{X_t^\theta\sim p_t^\theta}\left[\underbrace{\mathrm{TRF}(X_T^\theta)}_{h(X_T^\theta)}  + \int_0^T \left(\underbrace{\mathrm{IRF}(X_t^\theta)-\beta \sum_{y\neq X_t^\theta} Q_{T-t}(y,X_t^\theta)D_I(\E_{a_t^\theta\sim \pi^\theta(\cdot\mid t,X_t^\theta)}(a_t^\theta)_y \,\Vert\, s_{\theta_\mathrm{pre}}(T-t,X_t^\theta)_y)}_{\tilde{r}(t,X_t^\theta;\pi^\theta(\cdot\mid t,X_t^\theta))}\right)\rmd t \right].
\end{equation}
The objective \eqref{ELBO} coincides with the RL objective \eqref{RLobjective}. With this expression, for any $(t,x)\in [0,T]\times \cS$, we define the value function

\begin{align}
V^\theta(t,x)
={}&
\E_{X^\theta_s\sim p_s^\theta}
\Bigg[
\mathrm{TRF}(X_T^\theta)
+
\int_t^T
\Bigg(
\mathrm{IRF}(X_s^\theta)
\notag\\
&-\beta
\sum_{y\neq X_s^\theta}
Q_{T-s}(y,X_s^\theta)
D_I\!\left(
\E_{a_s^\theta\sim \pi^\theta(\cdot\,\mid\, s,X_s^\theta)}
(a_s^\theta)_y
\,\middle\Vert\,
s_{\theta_{\mathrm{pre}}}(T-s,X_s^\theta)_y
\right)
\Bigg)
\,\mathrm ds
\,\Bigg|\,
X_t^\theta=x
\Bigg].
\label{valuefn}
\end{align}

To design the fine-tuning algorithms, it suffices to inject the functions $R$, $r$, and $h$ in \eqref{diffusionrate} and \eqref{valuefn} into Algorithms~\ref{alg:PPO} and \ref{alg:GRPO}, and replace the initial distribution $\rho$ with $p_\mathrm{noise}$ and the initial parameter $\theta_0$ with $\theta_\mathrm{pre}$. Notably, our framework naturally accommodates arbitrary intermediate rewards,
does not require the reward function to be differentiable,
and avoids the non-differentiability issue of sampling from a categorical distribution when maximizing the reward (see \citet[Section~4.2]{wang2025fine}).

%% file: section4.tex
\section{Fine-Tuning Masked Diffusion Models}
\label{sec4}
Although our framework applies to general high-dimensional CTMC-based diffusion models as discussed in Section~\ref{sec3}, 
our focus will be on the widely adopted masked diffusion models (MDMs). 
In this setting, 
the score function has a direct correspondence with the denoiser, allowing the policy to be naturally parameterized as a probability distribution over the simplex.

\textbf{Notations.}
We use lowercase letters to denote scalars, and boldface lowercase letters to denote vectors. The $i$-th entry of a vector $\x$ is denoted by $\x^i$, or simply $x^i$ when the context is clear. We use $\x^{\backslash i}$ to refer to all dimensions of $\x$ except the $i$-th, and $\x^{\backslash i} \odot \hat{x}^i$ to denote a vector whose $i$-th dimension takes the value $\hat{x}^i$, while the other dimensions remain as $\x^{\backslash i}$. 
For a positive integer $n$, we denote
$\mathbf{1}_n \in \mathbb{R}^n$ as the column vector of ones, and $I_n \in \mathbb{R}^{n \times n}$ as the identity matrix. The notation $\e_j$ refers to a one-hot column vector with a $1$ in the $j$-th position. 
We use $\mathrm{Cat}(\cdot;\p)$ for a categorical distribution over a one-hot column vector with probabilities given by the column vector $\p$. $\mathbf{1}\{\cdot\}$ is the indicator function.

\subsection{Preliminaries on Masked Diffusion Models} Let $\cS=(\cV\cup\{\mathtt{m}\})^L$ be the state space, where $\cV:=\{1,\cdots,V\}$ is the vocabulary space, $\mathtt{m}$ is the special ``mask'' token, and $L$ is the sequence length. The continuous time horizon is set to $1$. We use superscript $i \in [L]$ for the sequence index, subscript $j \in \cV$ for the token vocabulary index, subscript $t$ for the current time (continuous-time, $t\in [0,1]$) or time step (discrete-time, $t\in [T]$), and subscript $g\in [G]$ for a specific sample trajectory in the GRPO algorithm.

\textbf{Forward Process.} The forward process is factorized as $p_{t|0}(\x_t\mid \x_0)=\prod_{i=1}^L p_{t|0}(x_t^i\mid x_0^i)$, where
$$
p_{t|0}(x_t^i\mid x_0^i)=\mathrm{Cat}(x_t^i;\alpha_t \e_{x_0^i}+(1-\alpha_t)\e_\mathtt{m}),\quad t\in[0,1].
$$
$\alpha_t$ is the noise scheduled to satisfy $\alpha_0\approx 1$ and $\alpha_1\approx 0$. A common choice is $\alpha_t=1-t$ \citep{nie2026large}, which we adopt throughout this work. The forward process interpolates the clean data and pure mask noise. The resulting forward rate matrix for each dimension is 
$$
Q_t^\mathrm{tok}=-\frac{1}{\alpha_t}\frac{\partial \alpha_t}{\partial t}(\mathbf{1}_{V+1}\cdot \e_\mathtt{m}^\top-I_{V+1})\in\R^{(V+1)\times (V+1)},
$$
and the full-dimensional rate matrix $Q_t=\oplus_{i=1}^L Q_t^\mathrm{tok} \in \R^{(V+1)^L\times (V+1)^L}$ is the Kronecker sum of $L$ dimension-wise rate matrices, with non-diagnal entries given by
$$
Q_t(\x,\x^{\backslash i}\odot \mathtt{m})=-\frac{1}{\alpha_t}\frac{\partial \alpha_t}{\partial t}=\frac{1}{1-t},\quad x^i\neq \mathtt{m}.
$$

\textbf{Reverse Process.} Given $\x_0\sim p_\mathrm{data}$, for $0\leq s\leq t$, the reverse transition probability is
$$
p_{s|t,0}(x_s^i\mid \x_t,\x_0)=\begin{cases}
    \mathrm{Cat}(x_s^i;\e_{x_t^i}), & x_t^i\neq \mathtt{m};\\
    \mathrm{Cat}(x_s^i;\frac{1-\alpha_s}{1-\alpha_t}\e_\mathtt{m}+\frac{\alpha_s-\alpha_t}{1-\alpha_t}\e_{x_0^i}),& x_t^i=\mathtt{m}.
\end{cases}
$$
Note that conditional on $\x_0$, if $x_t^i=\mathtt{m}$, then with probability $\frac{\alpha_s-\alpha_t}{1-\alpha_t}$, $x_t^i$ will jump to $x_0^i$ at time $s$; once $x_t^i$ is unmasked, it remains the same until $t=0$.
Thus, one can parameterize a denoiser
$\p_\theta(\x_t)$, which takes all the dimensions of $\x_t$ as the input and predicts masked tokens ($x_t^i=\mathtt{m},\ i\in [L]$) simultaneously. To be more specific, for any $i\in [L]$, a neural network $\p_\theta^{(i)}(\x_t)$, the $i$-th output of $\p_\theta(\x_t)$, outputs a probability distribution of $x_0^i$:
$$
\p_\theta^{(i)}(\x_t)=\begin{cases}
   \mathrm{softmax}(\f_\theta^{(i)}(\x_t)), &i\in[L]:x_t^i=\mathtt{m},\\
   \e_{x_t^i},&i\in[L]:x_t^i\neq \mathtt{m}
\end{cases}\in\Delta^{V-1},
$$
where $\Delta^{V-1}:=\{\x\in\R_+^V:\sum_{i=1}^V \x^i=1\}\subset \R^V$ is the $(V-1)$-dimensional simplex. For $i:x_t^i=\mathtt{m}$, $\f_\theta^{(i)}(\x_t)\in\R^V$ are the model logits and $\p_\theta^{(i)}(\x_t)\in \Delta^{V-1}$ is a probability vector with the $j$-th component, denoted by $p^{(i)}_\theta(\x_t)_j$ or $p_\theta^{(i)}(j\mid \x_t)$, predicting the conditional probability $p_{0|t}(x_0^i=j\mid \x_t)$ for all $j\in \cV$. Notably, by the identity \citep{shi2024simplified,ou2025your,hasan2026discrete}:
$$
s(t,\x_t)_{i,j}:=\frac{p_t(\x_t^{\backslash i}\odot j)}{p_t(\x_t)}=\delta\{\mathtt{m},j\}+\frac{\alpha_t}{1-\alpha_t}p_{0|t}(x_0^i=j\mid \x_t),\quad x_t^i=\mathtt{m},
$$
parameterizing the denoising probability is equivalent to parametrizing the score function up to some scalar:
$$
s_\theta(t,\x_t)_{i,j}=\frac{\alpha_t}{1-\alpha_t}p_\theta^{(i)}(\x_t)_j,\quad x_t^i=\mathtt{m}, \ j\in \cV.
$$
Here the denoiser network does not depend on the time $t$ because the partially unmasked sequence $\x_t$ implicitly contained the time information \citep{shi2024simplified,kim2025train,nie2026large}. In the remainder of this work, we always consider the denoiser parameterization $\p_\theta(\x_t)$ (which will be replaced by actions) instead of the score parameterization. Also, the reverse rate matrix is
$$
Q^\leftarrow_{1-t}(\x,\x^{\backslash i}\odot j)=-\frac{1}{\alpha_t}\frac{\partial \alpha_t}{\partial t}s(t,\x)_{i,j}=-\frac{1}{1-\alpha_t}\frac{\partial \alpha_t}{\partial t}p_{0|t}(x_0^i=j\mid \x)=\frac{1}{t}p_{0|t}(x_0^i=j\mid \x),\quad x^i=\mathtt{m},\ j\in\cV,
$$
Hence, the reverse process parametrized by $\p_\theta$ has rate matrix 
\begin{equation}\label{rate:reverse}
Q^\theta_t(\x,\x^{\backslash i}\odot j)=\frac{1}{1-t}p_\theta^{(i)}(\x)_j,\quad x^i=\mathtt{m},\ j\in \cV.
\end{equation}

\subsection{Post-Training}\label{sec:posttraining}

\textbf{Policy Parameterization.} 
In RL, an action $\a_t$ is sampled from a parameterized stochastic policy to encourage exploration: $\a_t\sim \pi^\theta(\cdot\mid t,\x_t)$ given the current time-state pair $(t,\x_t)$. Since the reverse masked diffusion dynamics \eqref{rate:reverse} is fully determined by the denoiser probability vector, we consider the action space $\mathcal{A}=(\Delta^{V-1})^L$. Then $\pi(\cdot\mid t,\x)$ is a distribution over the product simplex $(\Delta^{V-1})^L$. For simplicity, we let the parameterized policy $\pi^\theta$ factorize over all dimensions: 
$$
\pi^\theta(\a_t\mid t,\x_t)=\prod_{i=1}^L \pi^{\theta,(i)}(\a_t^{(i)}\mid t,\x_t),
$$
where $\pi^{\theta,(i)}(\cdot\mid t,\x_t)$ is a distribution over $\Delta^{V-1}$,
and $\a_t^{(i)}\in\Delta^{V-1}$ is a probability vector over $\cV$ for all $i\in [L]$. 
At time $t$, we sample from $\mathrm{Cat}(\cdot;\a_t^{(i)})$, instead of $\mathrm{Cat}(\cdot;\p_\theta^{(i)}(\x_t))$, to predict the clean token at the $i$-th position. 
So $\a_t^{(i)}$ can be viewed as performing exploration over the denoiser $\p_\theta^{(i)}(\x_t)$. 
As a result, we need to specify the dimension-wise policies $\pi^{\theta,(i)}(\cdot\mid t,\x)$ for all $i$. 

We emphasize that the policy parameterization must be chosen carefully. In particular, $\pi^\theta(\cdot\mid t,\x)$ should have a tractable likelihood, and the resulting probability ratio
$
\rho_t^\theta
:=
\pi^\theta(\a_t\mid t,\x_t)/\pi^{\theta_{\mathrm{old}}}(\a_t\mid t,\x_t)
$
should be amenable to stable optimization. Several policy parameterizations that perform full-logit exploration and induce distributions supported on the full simplex, including Dirichlet, temperature softmax, and logistic-normal policies, are summarized in Appendix~\ref{app2}, together with their practical implementation details. 
However, for masked dLLMs with a large vocabulary size $(|\cV|\approx 1.2\times10^5)$, it is inefficient to explore the full simplex: the learned denoiser distribution $\p_\theta$ is typically highly concentrated, with probabilities spanning several orders of magnitude and effectively residing on a much lower-dimensional region of the simplex, often near a small number of vertices. As discussed in Appendix~\ref{app2}, applying full-simplex exploration in this regime amounts to exploring the entire vocabulary, which can lead to unstable optimization due to the large vocabulary size.
Motivated by the empirical success of d1 \citep{zhao2025d}, 
we consider a policy family whose support is concentrated on the $V$ vertices of the simplex $\Delta^{V-1}$, assigning probability mass directly to the vertices rather than exploring its interior through a continuous density. 
Nevertheless, the CTRL perspective suggests that broader exploration of the denoiser distribution still be beneficial. Consequently, while vertex-supported policies are used for optimization, we retain simplex-level exploration during rollout generation in our experiments.

Following prior works \citep{zhao2025d,zhao2025diffpo,wang2026spg,wang2026revolutionizing,oba2026diffusion}, we focus on the GRPO algorithm for fine-tuning masked dLLMs to avoid training an expensive value network in PPO, though we also describe a possible value function approximation without training a network in Appendix~\ref{app3}. Let $(\x_t)_{t=0}^T$ be a decoding trajectory, where $\x_0$ is a prompt concatenated with mask tokens of the completion length $L$, $\x_T$ is the final output corresponding to the prompt, and $T \in\Z_+$ is the number of denoising steps. 
The trajectory can be generated by any inference strategy, 
while we adopt a semi-autoregressive strategy with low-confidence remasking and random sampling \citep{nie2026large,zhao2025d}. 
Specifically, at time step $t$, we first select the $k$ masked positions with the highest confidence score in the current block $B_t\subset [L]$, denoted by
\[
\mathcal{M}_t^{\mathrm{TopkConf}}
:=
\left\{
i\in B_t :
x_t^i=\mathtt{m},\;
c_t^i
\in
\operatorname{Topk}
\left(
\{c_t^i\}_{i\in B_t:x_t^i=\mathtt{m}}
\right)
\right\},\quad c_t^i
:=
\max_{j\in\mathcal V}
p_{\theta}^{(i)}(\mathbf{x}_t)_j,
\]
to unmask, and then sample according to $\mathrm{Cat}(\cdot;\a_t^{(i)})$ for all $i\in \cM_t^{\mathrm{TopkConf}}$:
\begin{equation}\label{our_sample}
\a_t^{(i)}\sim \pi^{\theta,(i)}(\cdot\mid t,\x_t)\quad \text{then}\quad \a_t^{(i)}=
\begin{cases}
    \e_j \quad \text{w.p.}\ p_\theta^{(i)}(\x_t)_j,\ j\in\cV, &\text{for}\ i\in\cM_t^{\mathrm{TopkConf}},\\
   \e_{x_t^i}, &\text{for}\ i\notin \cM_t^{\mathrm{TopkConf}}
\end{cases}\in \Delta^{V-1}.
\end{equation}

The corresponding intermediate and terminal reward functions are given by \eqref{ELBO}:
\begin{equation}\label{ourreward}
\tilde{r}(t,\x_t;\pi^\theta(\cdot\mid t,\x_t))=\mathrm{IRF}(\x_t)-\frac{\beta}{1-t}\sum_{i:x_t^i=\mathtt{m}} D_\mathrm{KL}(\p_\theta^{(i)}(\x_t) \,\Vert\, \p_{\theta_\mathrm{pre}}^{(i)}(\x_t))\quad \text{and}\quad h(\x_1)=\mathrm{TRF}(\x_1),
\end{equation}
and the probability ratio is
\begin{equation}\label{ourratio}
\rho_t^\theta=\frac{\pi^\theta(\a_t\mid t,\x_t)}{\pi^{\theta_\mathrm{old}}(\a_t\mid t,\x_t)}=\frac{\prod_{i\in\cM_t^\mathrm{TopkConf}}\pi^{\theta,(i)}(\a_t^{(i)}\mid t,\x_t)}{\prod_{i\in\cM_t^\mathrm{TopkConf}}\pi^{\theta_\mathrm{old},(i)}(\a_t^{(i)}\mid t,\x_t)}=\prod_{i\in\cM_t^\mathrm{TopkConf}}\frac{p_\theta^{(i)}(x_T^i\mid \x_t)}{p_{\theta_\mathrm{old}}^{(i)}(x_T^i\mid \x_t)}.
\end{equation}
Thus, the GRPO loss is
\begin{equation}\label{ourloss}
\cL(\theta)=\E_{(\x_{g,0},\x_{g,1},\cdots,\x_{g,T})\}_{g=1}^G\sim \pi^{\theta_\mathrm{old}}}\left[
\frac{1}{G}\sum_{g=1}^G \frac{1}{T} \sum_{t=1}^T \min\left(
\rho_{g,t}^\theta A_{g,t}, \mathrm{clip}\left(\rho_{g,t}^\theta,1-\epsilon,1+\epsilon\right) A_{g,t}
\right)
\right],
\end{equation}
where $T=L/k$, $\rho_{g,t}^\theta=\prod_{i\in\cM_{g,t}^\mathrm{TopkConf}}p_\theta^{(i)}(x_{g,T}^i\mid \x_{g,t})/p_{\theta_\mathrm{old}}^{(i)}(x_{g,T}^i\mid \x_{g,t})$, and the computation of $A_{g,t}$ is implied by \eqref{ourreward} and Algorithm~\ref{alg:GRPO}. 


\textbf{Practical Considerations.} Our CTRL-based GRPO loss in \eqref{ourloss} cannot be evaluated in a single causal forward pass over an entire denoising trajectory due to the evaluation of $\{\rho_{g,t}^\theta\}_{t\in [T]}$. In online GRPO-based methods for dLLMs, this leads to a fundamental trade-off: on one hand, accurate likelihood computation under a CTRL formulation and the incorporation of intermediate rewards to optimize intermediate steps both require conditioning on the intermediate state $\x_t$ at each time step $t$; on the other hand, this necessitates multiple forward passes along the trajectory, incurring computational and memory costs proportional to the prompt batch size, group size $G$, and number of denoising steps $T$, as all intermediate activations must be stored for backpropagation. Existing approaches make different compromises. 
For instance, d1 \citep{zhao2025d} computed the loss using a single forward pass by conditioning only on the perturbed prompt. DTRPO \citep{zhang2026dtrpo} estimated the trajectory probability via a single forward pass of a re-masked final state; however, it is designed for offline DPO-style objectives and is difficult to extend to online GRPO settings. d2 \citep{wang2026d2} utilized block composite likelihood as an estimator to reduce the number of forward passes. 

In practice, to reduce the number of required forward passes in our framework, we propose trajectory subsampling, which yields an unbiased estimator of the full GRPO loss in \eqref{ourloss}. At each optimizer update we draw a subset $\cT\subset [T]$ of size $N=|\cT|$ uniformly at random without replacement, the same subset being reused for the policy, old, and reference passes so the ratio \eqref{ourratio} is consistent. We then form the subsampled estimator
\begin{equation*}
\hat{\cL}_\cT(\theta)=\E_{(\x_{g,0},\x_{g,1},\cdots,\x_{g,T})\}_{g=1}^G\sim \pi^{\theta_\mathrm{old}}}\left[
\frac{1}{G}\sum_{g=1}^G \frac{1}{N} \sum_{t\in \cT} \min\left(
\rho_{g,t}^\theta A_{g,t}, \mathrm{clip}\left(\rho_{g,t}^\theta,1-\epsilon,1+\epsilon\right) A_{g,t}
\right)
\right],
\end{equation*}
which reduces the number of forward passes from $T$ to $N$ per rollout.
One can easily see that this estimator is unbiased by the fact that $\P(t\in \cT)=N/T$:
$\E_\cT [\hat{\cL}_\cT(\theta)]=\cL(\theta)$. When $N$ is too small, the model fails to reliably converge to competitive performance due to inaccurate loss estimation. Conversely, increasing $N$ substantially raises the number of forward passes, thereby reducing training efficiency. Our experiments indicate that $N=8$ provides a favorable trade-off, achieving performance comparable to $N=16$ and $N=32$ while requiring substantially fewer function evaluations. More theoretical derivation and practical implementation details including the ablation study on $N$ are provided in Appendix~\ref{app2-4}.

%% file: section5.tex
\section{Experiments}
\label{sec5}
We showcase our method for fine-tuning MDMs on two benchmarks: $2$-dimensional synthetic example with PPO, 
and higher-dimensional reasoning and coding dataset for dLLMs with GRPO.

\subsection{Synthetic Example}
\textbf{Setup.} We consider a synthetic example described in Discrete Adjoint Matching (DAM) \citep{so2026discrete}, where we fine-tune an MDM on a $90\times 90$ discrete
grid so that its terminal sampling distribution matches a prescribed
\textit{checkerboard} target (the first subfigure in Figure~\ref{fig:compare_checkerboard}). The target is the entropy-regularized optimum
$p^*\propto p^\mathrm{base}\,e^{h/\beta}$.
We solve the task with our proposed PPO algorithm. 

We use a two-dimensional discrete space
$\cX \;=\; \{\mathtt{m},1,2,\dots,90\}^2$, and a state is a pair $\x=(x^1,x^2)$ in
which each coordinate is either masked or holds a token. The MDM begins at the fully masked state $\mathbf{X}_0=(\mathtt{m},\mathtt{m})$ and unmasks exactly
one token per step, so there are $T=2$ denoising steps. The base/pretrained denoiser is uniform: at every masked
position it predicts a uniform distribution over the $V$ tokens, so the base
terminal distribution $p^\mathrm{base}_1$ is uniform on the grid and all target structure
is induced by the reward. We use the softmax exponential-temperature exploration policy described in Appendix~\ref{app2-2}. The PPO advantage is the value difference across one unmasking step. Because only the states reachable in the two-step process matter, we store an explicit logit table indexed by the context (position, other coordinate), where the other coordinate is a revealed token in $\{1,\cdots,90\}$ or the mask token, and we evaluate the value $V^\theta$ exactly by computing the surrogate's expectation over all $V$ tokens for each visited context in closed-form. Further implementation details are provided in Appendix~\ref{app4-1}.

\begin{figure}[ht]
\centering

\begin{subfigure}{0.158\textwidth}
    \includegraphics[width=\linewidth]{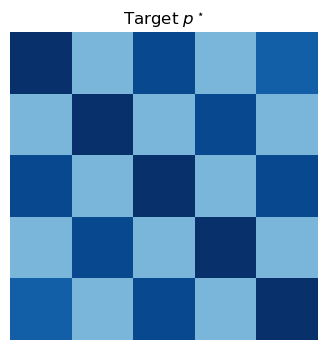}
\end{subfigure}%
\begin{subfigure}{0.158\textwidth}
    \includegraphics[width=\linewidth]{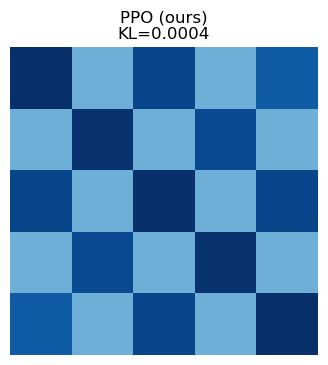}
\end{subfigure}%
\begin{subfigure}{0.158\textwidth}
    \includegraphics[width=\linewidth]{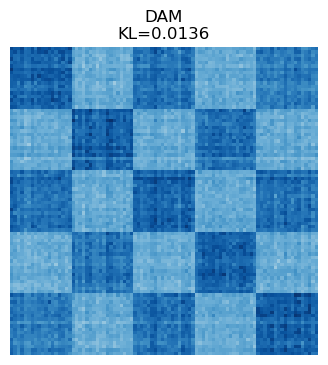}
\end{subfigure}%
\begin{subfigure}{0.158\textwidth}
    \includegraphics[width=\linewidth]{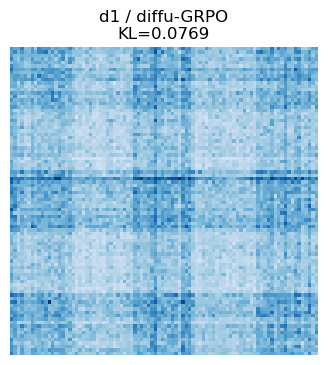}
\end{subfigure}%
\begin{subfigure}{0.2\textwidth}
    \includegraphics[width=\linewidth]{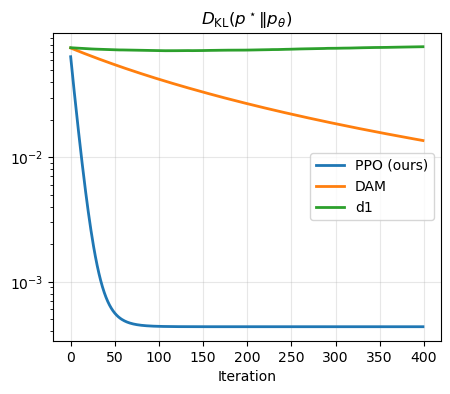}
\end{subfigure}%
\begin{subfigure}{0.2\textwidth}
    \includegraphics[width=\linewidth]{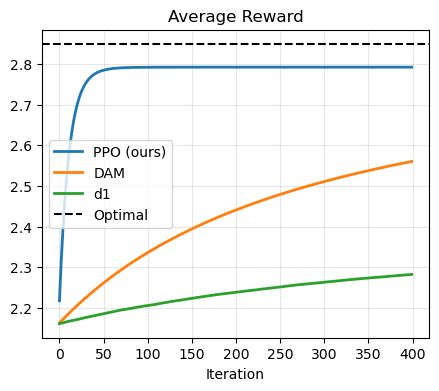}
\end{subfigure}
\caption{Three-way comparison. PPO and DAM converge toward $p^*$; d1 plateaus far above it.}
\label{fig:compare_checkerboard}
\end{figure}

\textbf{Results.} We select DAM \citep{so2026discrete} and d1 \citep{zhao2025d} as the baselines. The experimental results are presented in Figure~\ref{fig:compare_checkerboard}. PPO aligns most closely with $p^*$, both visually and numerically.
In this $L=2$ tabular setting, PPO uses
an exact critic, so it converges faster and tighter than DAM's sample-based adjoint.
d1 was designed for large-vocabulary reasoning with informative prompts, so this synthetic distribution-matching task isolates the objective rather than d1's intended use case.

\subsection{Mathematical Reasoning and Coding Tasks}
\textbf{Setup.} We evaluate our GRPO algorithm on four standard mathematical reasoning tasks: GSM8K \citep{cobbe2021training}, which consists of multi-step grade-school math word problems; MATH500 \citep{lightman2024let}, a collection of high-school competition mathematics problems; $3$-number Countdown, a combinatorial arithmetic game in which models must reach a target number using basic arithmetic operations on three given numbers; and $4\times 4$ Sudoku, which requires constraint satisfaction and systematic reasoning to complete the grid. 
We employ LLaDA-8B-Instruct \citep{nie2026large} without supervised fine-tuning as the base model, and adopt the experiment setting in d1 \citep{zhao2025d}. We choose d1 \citep{zhao2025d}, SPG with mixture \citep{wang2026spg}, and d2 \citep{wang2026d2} as the baselines: d1 is the first GRPO-based fine-tuning algorithm for dLLMs, and SPG and d2 are included because they are representative methods that outperform d1, while requiring multiple forward passes during training. For each method, we fine-tune a separate model on each task with completion lengths of $128$ and $256$ in 1k steps (due to the time cost), and evaluate each resulting
model at completion lengths of $128$ and $256$, accordingly. 

\begin{figure}[ht]
\centering
    \includegraphics[width=\linewidth]{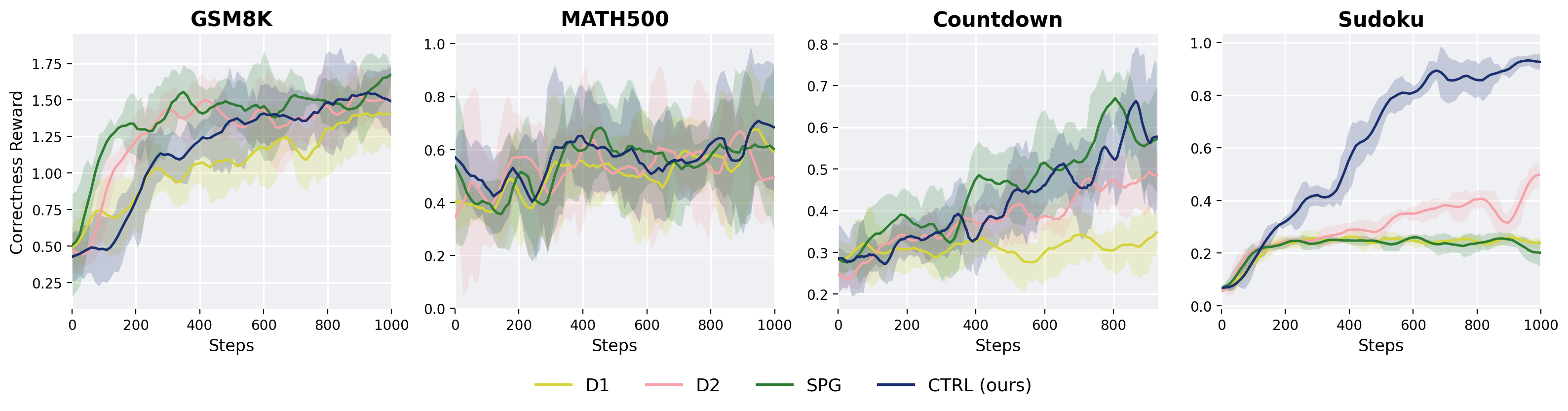}
\caption{Correctness reward dynamics of CTRL during RL training with a completion length of $128$, compared with d1, d2, and SPG.}\label{fig:reward_len128}
\end{figure}

\begin{table}[ht]
\centering
\begin{tabular}{l cc cc cc cc}
\toprule
& \multicolumn{2}{c}{\textbf{GSM8K}} & \multicolumn{2}{c}{\textbf{MATH500}} & \multicolumn{2}{c}{\textbf{Countdown}} & \multicolumn{2}{c}{\textbf{Sudoku}} \\
\cmidrule(lr){2-3} \cmidrule(lr){4-5} \cmidrule(lr){6-7} \cmidrule(lr){8-9}
\textbf{Model / Seq Len} & \textbf{128} & \textbf{256}  & \textbf{128} & \textbf{256}  & \textbf{128} & \textbf{256}  & \textbf{128} & \textbf{256}  \\
\midrule
LLaDA-8B-Instruct & 67.9 & 76.1 & 26.8 & 33.4 & 21.2 & 17.2  & 11.5 & 7.1 \\
D1 & 70.9 & 76.6 & 28.2 & 36.6 & 28.4 & 29.5 & 22.9 & 15.7 \\
D2 & 74.8 & 78.9 & 32.2 & 37.8 & 47.8 & 49.2 & 29.9 & 21.4 \\
SPG & 75.7 & 78.1 & 32.0 & 37.8 & \textbf{64.9} & 57.0 & 26.1 & 26.0 \\
\midrule
CTRL 
& \textbf{76.7} & \textbf{80.3} & \textbf{32.6} & \textbf{38.2} & 64.8 & \textbf{62.9} & \textbf{88.2} & \textbf{51.7}
\\
\bottomrule
\end{tabular}
\caption{Evaluation results in 1k training steps for each method on reasoning tasks. All evaluations are conducted using the evaluation code from d1 with $0$-shot prompting.}
\label{tab:dllm_evaluation}
\end{table}

\textbf{Results.} The correctness reward curves during the RL training with a completion length of $128$ are depicted in Figure~\ref{fig:reward_len128}. For the Sudoku experiments, we note that SPG uses $3$-shot prompting during training, whereas all four methods in our experiments are trained with $0$-shot prompting. Despite this more challenging setting, CTRL consistently converges to a correctness reward of $0.92$ on Sudoku in 1k training steps, whereas d2 converges much more slowly, and both d1 and SPG plateau at around $0.25$. We attribute the strong performance of CTRL primarily to its intermediate reward, which provides process supervision by penalizing outputs containing characters other than $1,2,3,4$ and thereby encourages the model to generate valid Sudoku entries (see Figure~\ref{fig:sudoku_interreward} in Appendix). In contrast, the other methods are more prone to exploiting the terminal reward: the model may either generate sequences containing more than $16$ numbers or randomly sample numbers from the set $\{1,2,3,4\}$. Under the terminal reward design \citep{zhao2025d}, which measures the proportion of correctly filled cells among those that were empty in the original puzzle, such reward-hacking behaviors can yield a correctness reward of $1/4$, explaining the observed plateau. Further experimental details including the specification of intermediate reward functions are provided in Appendix~\ref{app4-2}. 

Table~\ref{tab:dllm_evaluation} reports the best evaluation results achieved by each method over all checkpoints. For fair comparison, all evaluations are conducted with $0$-shot prompting, using the evaluation code from d1 with a batch size of $8$. Across nearly all experimental settings except Countdown with a completion length of $128$, CTRL achieves superior accuracy to all baselines, demonstrating its effectiveness for fine-tuning moderate-sized dLLMs with moderate sequence lengths. In particular, CTRL substantially outperforms the competing methods on Sudoku with a completion length of $128$; all methods collapse when the completion length is increased to $256$, suggesting that shorter sequence lengths are better suited to this task.

\textbf{Extension to Coding.} We further extend CTRL to coding
tasks, fine-tuning the LLaDA-8B-Instruct base model on the KodCodeLight-RL-10K dataset \citep{xu2025kodcode} and evaluating on HumanEval
and MBPP benchmarks. As shown in Table~\ref{tab:dllm_evaluation_coding},
CTRL consistently improves the accuracy on all benchmarks over the
baselines across different generation
lengths, demonstrating its strong ability in the coding domain.

\begin{table}[ht]
\centering
\begin{tabular}{l cc cc}
\toprule
& \multicolumn{2}{c}{\textbf{HumanEval}} & \multicolumn{2}{c}{\textbf{MBPP}} \\
\cmidrule(lr){2-3} \cmidrule(lr){4-5} 
\textbf{Model / Seq Len} & \textbf{128} & \textbf{256}  & \textbf{128} & \textbf{256}  \\
\midrule
LLaDA-8B-Instruct & 24.8 & 34.8 & 39.7 & 40.5\\
D1 &  29.3 & 39.0  & 42.0 &  45.5\\
D2 & 39.6 & 48.7 & 45.6 & 46.8 \\
SPG &  29.3 & 40.2  & 44.4 & 44.8  \\
\midrule
CTRL & \textbf{62.8} & \textbf{66.2} & \textbf{52.9} & \textbf{57.7}  \\
\bottomrule
\end{tabular}
\caption{Evaluation results in 1k training steps for each method on coding tasks. CTRL consistently outperforms all baselines on HumanEval and MBPP benchmarks.}
\label{tab:dllm_evaluation_coding}
\end{table}

%% file: section6.tex
\section{Conclusion and Future Work}
\label{sec6}
We introduce a theoretical framework of reinforcement learning in continuous time and discrete spaces, with applications to fine-tuning score-based discrete diffusion models,
especially masked diffusion models 
by setting the score/denoiser as the action.
Our framework naturally accommodates intermediate rewards into the RL objective, and leads to algorithms that do not require differentiable rewards. 
We demonstrate the effectiveness of our methods by
both solving low-dimensional entropy-regularized optimization problems and fine-tuning diffusion large language models on reasoning and coding tasks. Future work will focus on developing refined techniques to reduce the number of function evaluations associated with the computation of importance sampling probability ratios under the continuous-time RL framework,
and improve the computational efficiency of PPO-based RL post-training for large-scale problems.

%% file: appendix.tex
\begin{appendix}

\newpage
\section{Proofs}

\subsection{Proof of Proposition~\ref{prop:exploratory}}\label{app1-0}

\subsubsection{Heuristic Law-of-Large-Numbers Argument}

Analogous to the law-of-large-numbers argument in \cite{wang2020continuous}, we first give a heuristic proof.
Fix $(t,x)\in [0,T]\times \cS$. Consider $N$ independent copies of the controlled problem
run under the \emph{same} feedback policy $\pi(\cdot\mid t,x)$. In copy $i$, an action
$a^i_t\sim\pi(\cdot\mid t,x)$ is drawn independently, and over $[t,t+\Delta t)$ the copy
attempts a jump to $y\neq x$ with probability $R(t,x,a^i_t)_y\,\Delta t + o(\Delta t)$. Let
\[
J^i_t(y) := \mathbf{1}\{\text{copy } i \text{ jumps to } y \text{ over } [t,t+\Delta t)\}.
\]
Then $\E[J^i_t(y)\mid a^i_t] = R(t,x,a^i_t)_y\,\Delta t + o(\Delta t)$, and taking
expectation over $a^i_t\sim\pi$,
\[
\E[J^i_t(y)]
= \Big(\int_{\cA} R(t,x,a)_y\,\pi(a\mid t,x)\,\rmd a\Big)\Delta t + o(\Delta t)
= \tilde{R}(t,x;\pi(\cdot\mid t,x))_y\,\Delta t + o(\Delta t).
\]
Since the copies are i.i.d., the law of large numbers gives, as $N\to\infty$,
\[
\frac1N\sum_{i=1}^N J^i_t(y)
\xrightarrow{\text{a.s.}} \E[J^i_t(y)]
= \tilde{R}(t,x;\pi(\cdot\mid t,x))_y\,\Delta t + o(\Delta t).
\]
Thus the empirical infinitesimal jump frequency from $x$ to $y$, averaged over the
externally randomized copies, equals the exploratory rate $\tilde{R}(t,x;\pi(\cdot\mid t,x))_y$.

\subsection{Rigorous Proof}

Now we give a rigorous proof by showing that both processes solve the same martingale problem,
which for time-inhomogeneous CTMCs on a discrete space has a unique solution.

The controlled process $X^\pi$ is built by: in state $x$ at time $t$, first sample
$a\sim\pi(\cdot\mid t,x)$, then evolve infinitesimally under generator $\mathcal{L}^a$. To
identify the law of $X^\pi$ we compute the infinitesimal evolution of
$\E[f(t,X^\pi_t)]$, i.e.\ the generator obtained after averaging over the action.

Let $f:[0,T]\times\cS\to\R$ be bounded. By~\eqref{pidynamics}, conditioning on the
state $x^\pi_t=x$ and the sampled action $a^\pi_t=a$,
\[
\E\big[f(t+\Delta t,X^\pi_{t+\Delta t})\mid x^\pi_t=x,\,a^\pi_t=a\big]
= \sum_{y\in\cS}\Big(\delta\{x,y\} + R(t,x,a)_y\,\Delta t + o(\Delta t)\Big) f(t+\Delta t,y).
\]
Using
$\sum_y\delta\{x,y\}f(t+\Delta t,y) = f(t+\Delta t,x)
= f(t,x) + \frac{\partial f}{\partial t}(t,x)\,\Delta t + o(\Delta t)$
and
$\inner{R(t,x,a)}{f(t+\Delta t, \cdot)}=\inner{R(t,x,a)}{f(t,\cdot)} + o(1)$,
this becomes
\[
f(t,x)
+ \Big(\tfrac{\partial f}{\partial t}(t,x) + \inner{R(t,x,a)}{f(t,\cdot)}\Big)\Delta t
+ o(\Delta t)
= f(t,x) + \mathcal{L}^a f(t,x)\,\Delta t + o(\Delta t).
\]
Now average over $a\sim\pi(\cdot\mid t,x)$. By linearity of $\mathcal{L}^a$ in $R(t,x,a)$,
and finiteness of $\cS$ (so the sum over $y$ and the expectation over $a$ may be
interchanged),
\[
\E_{a\sim\pi(\cdot\mid t,x)}\big[\mathcal{L}^a f(t,x)\big]
= \frac{\partial f}{\partial t}(t,x)
+ \Big\langle\, \underbrace{\E_{a\sim\pi(\cdot\mid t,x)} R(t,x,a)}_{=\,\tilde{R}(t,x;\pi(\cdot\mid t,x))},\; f(t,\cdot)\,\Big\rangle.
\]
Define the averaged generator
\[
\tilde{\cL} f(t,x) := \E_{a\sim\pi(\cdot\mid t,x)}\big[\mathcal{L}^a f(t,x)\big]
= \frac{\partial f}{\partial t}(t,x) + \inner{\tilde{R}(t,x;\pi(\cdot\mid t,x))}{f(t,\cdot)}.
\]
Therefore, we have
\begin{equation}\label{eq:star}
\E\big[f(t+\Delta t,X^\pi_{t+\Delta t})\mid x^\pi_t=x\big]
= f(t,x) + \tilde{\cL} f(t,x)\,\Delta t + o(\Delta t).
\end{equation}

On the other hand, the exploratory process $\tilde{X}^\pi$ jumps directly with rate $\tilde{R}(t,x;\pi(\cdot\mid t,x))$, with no external
action sampling. Repeating the identical one-step computation
using~\eqref{RLprocess},
\begin{equation}\label{eq:starstar}
\E\big[f(t+\Delta t,\tilde{X}^\pi_{t+\Delta t})\mid \tilde{X}^\pi_t=x\big]
= f(t,x) + \tilde{\cL} f(t,x)\,\Delta t + o(\Delta t).
\end{equation}

Comparing~\eqref{eq:star} and~\eqref{eq:starstar}, both processes have the same
infinitesimal generator $\tilde{\cL}$, and thus have equal one-step transition rates.

To conclude equality of the full path laws, we invoke uniqueness of the martingale problem. Define
\[
M^f_t := f(t,X_t) - f(0,X_0) - \int_0^t \tilde{\cL} f(s,X_s)\,\rmd s
\]
for bounded $f:[0,T]\times\cS\to\R$.
A process $X$ with $X_0\sim\rho$ \emph{solves the martingale problem for $(\tilde{\cL},\rho)$} if
$M^f_t$ is a martingale for every bounded $f$. Hence, it suffices to prove that both $X^\pi$ and $\tilde{X}^\pi$ solve this martingale problem. For $\tilde{X}^\pi$, this is the standard Dynkin's formula (see, e.g., \citet[Lemma~3]{benton2024denoising}) applied to the CTMC with rate
$\tilde{R}$, so $M^f_t(\tilde{X}^\pi)$ is a martingale by construction. For $X^\pi$, we must check the controlled process is also a martingale solution for the averaged generator $\tilde{\cL}$. Let $\cF_t=\sigma(X^\pi_s:s\le t)$. For $u<t$, we have
\[
\E\big[M^f_t(X^\pi) - M^f_u(X^\pi)\mid\cF_u\big]
= \E\Big[f(t,X^\pi_t) - f(u,X^\pi_u) - \int_u^t \tilde{\cL} f(s,X^\pi_s)\,\rmd s \,\Big|\, \cF_u\Big].
\]
Partition $[u,t]$ into mesh $\Delta s$. By~\eqref{eq:star}, the conditional increment of $f$
over each subinterval, after averaging the external action $a^\pi_s$ (sampled independently
of the past given the current state), is
\[
\E\big[f(s+\Delta s,X^\pi_{s+\Delta s}) - f(s,X^\pi_s)\mid\cF_s\big]
= \tilde{\cL} f(s,X^\pi_s)\,\Delta s + o(\Delta s).
\]
Crucially, the averaging in~\eqref{eq:star} uses the Markov property: the action
$a^\pi_s$ at time $s$ depends on the past only through $X^\pi_s$, so conditioning on $\cF_s$
and then on $a^\pi_s$ collapses to the state-conditional average
$\E_{a\sim\pi(\cdot\mid s,X^\pi_s)}$, yielding $\tilde{\cL}$. Summing over the partition and letting
$\Delta s\to0$ (dominated convergence is valid since $\cS$ is finite and the rates are
bounded), the telescoping sum gives
\[
\E\big[f(t,X^\pi_t) - f(u,X^\pi_u)\mid\cF_u\big]
= \E\Big[\int_u^t \tilde{\cL} f(s,X^\pi_s)\,ds \,\Big|\, \cF_u\Big],
\]
and thus $\E[M^f_t - M^f_u\mid\cF_u]=0$. Hence, $X^\pi$ also solves the martingale problem for
$(\tilde{\cL},\rho)$.

Finally, for a CTMC on a finite state space with
measurable, bounded time-dependent rate matrix $\tilde{R}(t,\cdot;\pi)$, the martingale problem
for $(\tilde{\cL},\rho)$ is well posed: there exists a unique law on $D([0,T],\cS)$ solving it \citep[Theorem 4.1]{ethier2009markov}. Since both $X^\pi$ and $\tilde{X}^\pi$ solve the same well-posed martingale problem $(\tilde{\cL},\rho)$, by
uniqueness, they induce the same law on path space:
\[
\mathrm{Law}\big((X^\pi_t)_{t\in[0,T]}\big) = \mathrm{Law}\big((\tilde{X}^\pi_t)_{t\in[0,T]}\big).
\]
$\hfill\square$

\subsection{Proof of Lemma~\ref{lemma:FK}}\label{app1-1}
\textit{Proof of Lemma~\ref{lemma:FK}.}
Note that 
\begin{align*}
\E_{a\sim \pi(\cdot\mid t,x)}(\cL^a v(t,x)+r(t,x,a))={}&\int_\cA (\cL^a v(t,x)+r(t,x,a))\pi(a\mid t,x)\,\rmd a\\
={}&\frac{\partial v}{\partial t}(t,x)+\langle \tilde{R}(t,x;\pi(\cdot\mid t,x)), v(t,\cdot)\rangle+\tilde{r}(t,x;\pi(\cdot\mid t,x))\\
={}&\tilde{\cL}v(t,x)+\tilde{r}(t,x;\pi(\cdot\mid t,x)),
\end{align*}
where we denote $\tilde{\cL}v(t,x)=\E_{a\sim \pi(\cdot\mid t,x)} \cL^a v(t,x)$. Then Lemma~\ref{lemma:FK} follows from \citet[Theorem 5]{benton2024denoising}. $\hfill\square$

\subsection{Proof of Theorem~\ref{lemma:PG}}\label{app1-2}
\textit{Proof of Theorem~\ref{lemma:PG}.} The proof is similar to that of Theorem 5 in \citet{jia2022policy}. By Lemma~\ref{lemma:FK}, we replace $v(t,x)$ by $V^\theta(t,x)$ and obtain that for any $(t,x)\in [0,T]\times \cS$,
$$
\int_\cA (\cL^a V^\theta(t,x)+r(t,x,a))\pi^\theta(a\mid t,x)\,\rmd a=0.
$$
Taking derivative w.r.t. $\theta$ on both sides, we have
\begin{align*}
    &\int_\cA \left[(\cL^a \nabla_\theta V^\theta(t,x))\pi^\theta(a\mid t,x) + (\cL^a V^\theta(t,x)+r(t,x,a))\nabla_\theta \pi^\theta(a\mid t,x) \right]\rmd a\\
    ={}& \int_\cA \left[\left(\cL^a \nabla_\theta V^\theta(t,x) + (\cL^a V^\theta(t,x)+r(t,x,a))\nabla_\theta \log\pi^\theta(a\mid t,x)\right)\pi^\theta(a\mid t,x) \right]\rmd a=0.
\end{align*}
By again Lemma~\ref{lemma:FK}, we have
$$
\nabla_\theta V^\theta(t,x)=\E\left[\int_t^T (\cL^a V^\theta(s,x)+r(s,x,a))\nabla_\theta \log\pi^\theta(a\mid s,x) \,\rmd s\mid X_t^\theta=x  \right].
$$
By letting $t=0$ and integrating $x$, we conclude the proof. $\hfill\square$

\subsection{Proof of Proposition~\ref{prop:qfunction}}\label{app1-3}
\textit{Proof of Proposition~\ref{prop:qfunction}.}
The proof is similar to that of Theorem~3 in \cite{jia2023q}. We have
\begin{align*}
&Q_{\Delta t}(t,x,a;\pi)\\
={}&\E\left[
\int_t^{t+\Delta t}r(s,X_s^a,a)\,\rmd s+\E\left[\int_{t+\Delta t}^T r(s,X_s^\pi,a_s^\pi)\,\rmd s+h(X_T^\pi)\mid X_{t+\Delta t}^a\right]\mid X_t^{\hat{\pi}}=x
\right] \\
={}&\E\left[
\int_t^{t+\Delta t}r(s,X_s^a,a)\,\rmd s+V(t+\Delta t,X_{t+\Delta t}^a;\pi)\mid X_t^{\hat{\pi}}=x
\right] \\
={}&\E\left[
\int_t^{t+\Delta t}r(s,X_s^a,a)\,\rmd s+V(t+\Delta t,X_{t+\Delta t}^a;\pi)-V(t,X_t^a;\pi)\mid X_t^{\hat{\pi}}=x
\right] +V(t,x;\pi)\\
={}&\E\left[
\int_t^{t+\Delta t}r(s,X_s^a,a)\,\rmd s+\int_t^{t+\Delta t}\cL^a V(s,X_s^a;\pi)\,\rmd s\mid X_t^{\hat{\pi}}=x
\right] +V(t,x;\pi)\\
={}& V(t,x;\pi)+q(t,x,a;\pi)\cdot\Delta t+o(\Delta t),
\end{align*}
where the second to the last equality follows from Dynkin's formula \citep[Lemma 3]{benton2024denoising} and the last equality is because of the definition of $q$ and the approximation of the integral. $\hfill\square$

\subsection{Proof of Proposition~\ref{prop:diffusion}}\label{app:prop3}
\textit{Proof of Proposition~\ref{prop:diffusion}.}
For the off-diagonal target $y=i\neq x$, the only column contributing to row $y$ is column
$i$, giving $(\Phi a)_y = \Phi_{y,y}\,a_y = Q_{T-t}(y,x)\,a_y$, matching the embedding. For the
diagonal row $y=x$, $(\Phi a)_x = \sum_{i\neq x}\Phi_{x,i}\,a_i = -\sum_{i\neq x}Q_{T-t}(i,x)\,a_i
= -\sum_{y\neq x}R(t,x,a)_y$, matching the conservative diagonal. Each column $i$ has exactly
two nonzero entries, $+Q_{T-t}(i,x)$ in row $i$ and $-Q_{T-t}(i,x)$ in row $x$, so its column
sum is $Q_{T-t}(i,x)-Q_{T-t}(i,x)=0$.
$\hfill\square$

\subsection{Proof of Theorem~\ref{KLdivergence}}\label{app1-4}
\textit{Proof of Theorem~\ref{KLdivergence}.} The proof follows by applying Lemma~1 in \cite{zhang2025convergence}, where two CTMCs with rate matrices $\tilde{R}(t,x;\pi^\theta(\cdot\mid t,x))_y=Q_{T-t}(y,x)\E_{a\sim\pi^\theta(\cdot\mid t,x)}a_y$ and $\tilde{R}(t,x;\pi^{\theta_\mathrm{pre}}(\cdot\mid t,x))=Q_{T-t}(y,x) s_{\theta_\mathrm{pre}}(T-t,x)_y$ are applied to Girsanov's theorem.
$\hfill\square$

\section{Policy Parameterization Choices}
\subsection{Policy over the Full Simplex}\label{app2}
In this section, we present several policy parameterizations that explore the full simplex (and therefore full-vocabulary exploration) and are applicable to fine-tuning MDMs when the vocabulary size $V$ is reasonably small.

\subsubsection{Dirichlet Policy}

We first introduce an unbiased policy by leveraging Dirichlet distribution. Consider
$$
\a_t^{(i)}\sim \pi^{\theta,(i)}(\cdot\mid t,\x_t):=
\begin{cases}
    \mathrm{Dir}(\cdot;k_t\cdot \p^{(i)}_\theta(\x_t)), &i\in[L]: x_t^i=\mathtt{m},\\
   \delta_{\mathrm{Cat}(\cdot;\e_{x_t^i})}, &i\in[L]: x_t^i\neq \mathtt{m}
\end{cases}\in\Delta(\Delta^{V-1}).
$$
Here, $\p_\theta^{(i)}(\x_t)$ determines the mean of the Dirichlet distribution and the time-dependent positive scalar $k_t\in\R_+$ controls its variance. Then we have $\E_{\a_t^{(i)}\sim \pi^{\theta,(i)}(\cdot|t,\x_t)}\a_t^{(i)}=\p_\theta^{(i)}(\x_t)\in\Delta^{V-1}$. 

\begin{remark}[Interpretation and Intuition of Using Dirichlet Distribution] The Dirichlet distribution is defined over the simplex as the conjugate prior of the categorical distribution, playing a role somewhat
analogous to that of the Gaussian distribution in continuous  diffusion models \citep{richemond2022categorical}.  
Note that given $\alpha\in\R_+^V$, for any $k\in \R_+$, it holds that $\E (\text{Dir}(\cdot;\alpha))=\E(\text{Dir}(\cdot;k\cdot \alpha))=(\alpha_1/\alpha_0,\cdots,\alpha_V/\alpha_0)^\top$ with $\alpha_0:=\sum_{j=1}^V \alpha_j$. Thus, we can tune $k$ to control the variance of $\pi^\theta$ while preserving the mean. Further let $\tilde{\alpha}_j:=\alpha_j/\alpha_0$, then  $\text{Var}(\text{Dir}(\cdot;\alpha))=(\tilde{\alpha}_j(1-\tilde{\alpha}_j)/(1+\alpha_0))_{j\in [V]}$, from which we can see that if fixing the mean ($\p_\theta(\x_t)$), tuning $k$ ($k_t$) small leads to a large variance and thus adds more exploration; in contrast, tuning $k$ large reduces stochasticity. On the other hand, given a fixed $\alpha_0$ ($k_t$), if $\tilde{\alpha}_j$ ($p_\theta(\x_t)_j$) is too large or small (close to $1$ or $0$), meaning that the model gives a very high or low confidence for the token $j$, then the variance of $a_j$ is small and thus the predicted probability value of the token $j$ is more deterministic; else if $\tilde{\alpha}_j$ lies in the middle of $[0,1]$, meaning that the model is not confident enough on the outcome of token $j$, then the variance of $a_j$ is large and more exploration and stochasticity are injected in the dynamics by the algorithm.
\end{remark}

\textbf{Value and Reward Functions.} For Dirichlet policy, the value function in \eqref{valuefn} becomes
\begin{align*}
&V^\theta(t,\x)\\
={}&\E_{\mathbf{X}_s\sim p_s^\theta}\left[\mathrm{TRF}(\mathbf{X}_1)+ \int_t^1  \left(\mathrm{IRF}(\mathbf{X}_s)-\frac{\beta}{1-s} \sum_{i:X_s^i=\mathtt{m}}\sum_{j\in\cV} D_I(p_\theta^{(i)}(\mathbf{X}_s)_j \,\Vert\, p_{\theta_\mathrm{pre}}^{(i)}(\mathbf{X}_s)_j)\right)\rmd s \,\Bigg|\, \mathbf{X}_t=\x  \right]\\
={}&\E_{\mathbf{X}_s\sim p_s^\theta}\left[\mathrm{TRF}(\mathbf{X}_1)+\int_t^1 \left(\mathrm{IRF}(\mathbf{X}_s)- \frac{\beta}{1-s} \sum_{i:X_s^i=\mathtt{m}} D_\mathrm{KL}(\p_\theta^{(i)}(\mathbf{X}_s) \,\Vert\, \p_{\theta_\mathrm{pre}}^{(i)}(\mathbf{X}_s))\right)\rmd s \,\Bigg|\, \mathbf{X}_t=\x  \right].
\end{align*}

Consequently, for $t \in [0,1)$, the intermediate and terminal reward functions are specified as 
$$
\tilde{r}(t,\x_t;\pi^\theta(\cdot\mid t,\x_t))=\mathrm{IRF}(\x_t)-\frac{\beta}{1-t}\sum_{i:x_t^i=\mathtt{m}} D_\mathrm{KL}(\p_\theta^{(i)}(\x_t) \,\Vert\, \p_{\theta_\mathrm{pre}}^{(i)}(\x_t))\quad \text{and}\quad h(\x_1)=\mathrm{TRF}(\x_1).
$$
We can therefore define
\begin{align*}
r(t,\x_t,\a_t):={}&\mathrm{IRF}(\x_t)+\frac{\beta}{1-t}\sum_{i:x_t^i=\mathtt{m}}\left(D_\mathrm{KL}(\a_t^{(i)}\,\Vert\,\p_\theta^{(i)}(\x_t))-D_\mathrm{KL}(\a_t^{(i)}\,\Vert\, \p_{\theta_\mathrm{pre}}^{(i)}(\x_t)) \right) \\
={}&\mathrm{IRF}(\x_t)-\frac{\beta}{1-t}\sum_{i:x_t^i=\mathtt{m}} \sum_{j\in\cV}(\a_t^{(i)})_j\log\frac{p_\theta^{(i)}(\x_t)_j}{p_{\theta_\mathrm{pre}}^{(i)}(\x_t)_j},\quad \a_t\sim \pi^\theta(\cdot\mid t,\x_t).
\end{align*}
We can see that a large $r$ corresponds to a large $D_{\mathrm{KL}}(\a_t^{(i)} \,\Vert\, \p_\theta^{(i)}(\x_t))$, encouraging exploration beyond the current mean $\p_\theta$, while maintaining a small $D_{\mathrm{KL}}(\a_t^{(i)} \,\Vert\, \p_{\theta_{\mathrm{pre}}}^{(i)}(\x_t))$, thereby keeping the exploration close to the pretrained model $\p_{\theta_{\mathrm{pre}}}$.

\textbf{$q$-Function.}
For $\a_t\sim \pi^\theta(\cdot\mid t,\x_t)$, we can write 
$$
q(t,\x_t,\a_t;\pi^\theta(\cdot\mid t,\x_t))=\mathrm{IRF}(\x_t)+\frac{1}{1-t}\sum_{i:x_t^i=\mathtt{m}}\sum_{j\in\cV} (\a_t^{(i)})_j\left(V^{\theta}(t,\x_t^{\backslash i}\odot j)-V^{\theta}(t,\x_t)-\beta\log\frac{p_\theta^{(i)}(\x_t)_j}{p_{\theta_\mathrm{pre}}^{(i)}(\x_t)_j}\right).
$$

\textbf{Probability Ratio.} To sample $\a\sim \pi^\theta(\cdot\mid t,\x)=\mathrm{Dir}(\cdot;\alpha_\theta(t,\x))$ where $\alpha_\theta(t,\x)\in\R^V_+$, one can first sample $V$ independent Gamma random variables $g_j\sim \Gamma(\alpha_\theta(t,x)_j,1)$, and then set $a_j=g_j/\sum_{k=1}^V g_k$ for all $j\in [V]$. Since the density of Dir($\cdot;\alpha$) is
$$
\mathrm{Dir}(a;\alpha)=\frac{\Gamma(\sum_{j=1}^V \alpha_j)}{\prod_{j=1}^V \Gamma(\alpha_j)}\prod_{j=1}^V a_j^{\alpha_j-1},
$$
we have
$$
\log \mathrm{Dir}(a;\alpha)=\sum_{j=1}^V (\alpha_j-1)\log a_j-\sum_{j=1}^V \log\Gamma(\alpha_j)+\log \Gamma(\sum_{j=1}^V \alpha_j).
$$
Then we obtain the probability ratio
\begin{align*}
\frac{\pi^\theta(a)}{\pi^{\theta_\mathrm{old}}(a)}={}&\exp(\log \pi^\theta(a)-\log \pi^{\theta_\mathrm{old}}(a))\\
={}&\exp\left(\sum_{j=1}^V (\alpha_j^\theta-\alpha_j^{\theta_\mathrm{old}})\log a_j-\sum_{j=1}^V \log\frac{\Gamma(\alpha_j^\theta)}{\Gamma(\alpha_j^{\theta_\mathrm{old}})}+\log\frac{\Gamma(\sum_{j=1}^V \alpha_j^\theta)}{\Gamma(\sum_{j=1}^V \alpha_j^{\theta_\mathrm{old}})}\right).
\end{align*}
Specific to MDMs, we have
\begin{align}
    \rho^\theta_t={}&\frac{\pi^\theta(\a_t\mid t,\x_t)}{\pi^{\theta_\mathrm{old}}(\a_t\mid t,\x_t)}\notag\\
    ={}&\prod_{i\in B_t: x_t^i=\mathtt{m}} \frac{\pi^{\theta,(i)}(\a_t^{(i)}\mid t,\x_t)}{\pi^{\theta_{\mathrm{old}},(i)}(\a_t^{(i)}\mid t,\x_t)}\notag\\
    ={}& \prod_{i\in B_t: x_t^i=\mathtt{m}}\exp\left(\sum_{j\in\cV} k_t(p_\theta^{(i)}(\x_t)_j-p_{\theta_\mathrm{old}}^{(i)}(\x_t)_j)\log (\a_t^{(i)})_j-\sum_{j\in\cV} \log\frac{\Gamma(k_t\cdot p_\theta^{(i)}(\x_t)_j)}{\Gamma(k_t\cdot p_{\theta_\mathrm{old}}^{(i)}(\x_t)_j)}\right)\notag\\
     ={}& \exp\left(\sum_{i\in B_t: x_t^i=\mathtt{m}}\sum_{j\in\cV} \left(k_t(p_\theta^{(i)}(\x_t)_j-p_{\theta_\mathrm{old}}^{(i)}(\x_t)_j)\log (\a_t^{(i)})_j-\log\frac{\Gamma(k_t\cdot p_\theta^{(i)}(\x_t)_j)}{\Gamma(k_t\cdot p_{\theta_\mathrm{old}}^{(i)}(\x_t)_j)}\right)\right),\label{rho:Dirichlet}
\end{align}
where $B_t \subset [L]$ is the sequence indices of the current block corresponding to the current time step $t$. 


\subsubsection{Temperature Softmax Policy}\label{app2-2}

Although Dirichlet policy is unbiased, the probability ratio has a complicated expression and suffers from high numerical instability, leading to unstable training. In dLLMs, people typically sample rollouts according to $\mathrm{softmax}(\f_\theta^{(i)}(\x_t)/\tau)$ \citep{zhao2025d}, where the temperature $\tau\geq 0$ control the sampling stochasticity (and exploration), and higher temperature leads to higher stochasticity: if $\tau=0$, it reduces to perform deterministic greedy sampling by unmasking the token with the largest logit value, $\argmax_{j\in\cV} f_\theta^{(i)}(\x_t)_j$; if $\tau=1$, then it is vanilla sampling according to the softmax applied to the model logits $\f_\theta^{(i)}(\x_t)$; and if $\tau=+\infty$, it corresponds to sampling from the uniform distribution over $\cV$, which is totally random. \cite{zhao2025d} chose $\tau=0.9$ for training. Here, since we view the probability simplex $\Delta^{V-1}$ as the action space, dividing the logits by a single $\tau$ followed by applying softmax cannot make the resulting probability vectors cover the whole high-dimensional simplex $\Delta^{V-1}$ by varying a one-dimensional $\tau$. Thus, we turn to consider \textit{dimension-aware} temperature softmax exploration:

$$
\a_t^{(i)}\sim \pi^{\theta,(i)}(\cdot\mid t,\x_t)\quad \text{then}\quad \a_t^{(i)}=
\begin{cases}
    \mathrm{softmax}\left((f_\theta^{(i)}(\x_t)_j/\tau_j)_{j\in \cV}\right), &i\in[L]: x_t^i=\mathtt{m},\\
   \mathrm{Cat}(\cdot;\e_{x_t^i}), &i\in[L]: x_t^i\neq \mathtt{m}
\end{cases}\in \Delta^{V-1}.
$$

where $\tau_j \overset{\text{i.i.d.}}{\sim} p,\ j\in\cV$ and $p$ is a continuous probability distribution over $(0,+\infty)$ which is to be picked. Here, each dimension of the logit $f^{(i)}_\theta(\x_t)_j$ is perturbed by a different but i.i.d. $\tau_j$, therefore enhancing the degree of freedom of $\tau$ to $V$ while avoiding increasing much complexity by specifying only a one-dimensional probability density $p$, in contrast to a complicated $(V-1)$-dimensional Dirichlet density. The dimension-aware temperature softmax approach also allows us to give an estimation of the probability ratio. Suppose $\a_t^{(i)}= \mathrm{softmax}((f_{\theta_\mathrm{old}}^{(i)}(\x_t)_j/\tau_j)_{j\in\cV})$ for some $\{\tau_j\}_{j\in\cV}$. One can see that there exists $\tau'_j=f_{\theta}^{(i)}(\x_t)_j\cdot  \tau_j/f_{\theta_\mathrm{old}}^{(i)}(\x_t)_j$ such that $\a_t^{(i)}=\mathrm{softmax}((f_{\theta_\mathrm{old}}^{(i)}(\x_t)_j/\tau_j)_{j\in\cV})=\mathrm{softmax}((f_{\theta}^{(i)}(\x_t)_j/\tau'_j)_{j\in\cV})$. Hence, we obtain that

\begin{align}
    \rho^\theta_t=\frac{\pi^\theta(\a_t\mid t,\x_t)}{\pi^{\theta_\mathrm{old}}(\a_t\mid t,\x_t)}
    ={}&\prod_{i\in B_t: x_t^i=\mathtt{m}} \frac{\pi^{\theta,(i)}(\a_t^{(i)}\mid t,\x_t)}{\pi^{\theta_{\mathrm{old}},(i)}(\a_t^{(i)}\mid t,\x_t)}\notag\\
    \approx{}& \prod_{i\in B_t: x_t^i=\mathtt{m}} \prod_{j\in\cV}\frac{f_\theta^{(i)}(\x_t)_j}{f_{\theta_\mathrm{old}}^{(i)}(\x_t)_j}\cdot\frac{p(\tau_j')}{p(\tau_j)}\notag
    \\
     ={}& \prod_{i\in B_t: x_t^i=\mathtt{m}}\prod_{j\in\cV} \frac{f_\theta^{(i)}(\x_t)_j}{f_{\theta_\mathrm{old}}^{(i)}(\x_t)_j}\cdot\frac{p(f_{\theta}^{(i)}(\x_t)_j\cdot  \tau_j/f_{\theta_\mathrm{old}}^{(i)}(\x_t)_j)}{p(\tau_j)},\label{rho:temperature}
\end{align}
where the approximation arises because the softmax function is not injective. One can also consider adaptive $\tau_t$, where $(\tau_t)_j \overset{\text{i.i.d.}}{\sim} p_t,\ j\in\cV$ for a time-dependent continuous distribution $p_t$. In particular, if we choose time-dependent $p_t$ as $\mathrm{Exp}(\lambda_t)$ which we refer to as exponential-temperature softmax policy, then 
\begin{align*}
    \rho^\theta_t\approx{}& \exp\left(\sum_{i\in B_t: x_t^i=\mathtt{m}}\sum_{j\in\cV}\left( \log\frac{f_\theta^{(i)}(\x_t)_j}{f_{\theta_\mathrm{old}}^{(i)}(\x_t)_j}+ \lambda_t(\tau_t)_j \left( 1-\frac{f_{\theta}^{(i)}(\x_t)_j}{f_{\theta_\mathrm{old}}^{(i)}(\x_t)_j}\right) \right)\right).
\end{align*}

\subsubsection{Logistic-Normal Policy}
Although the temperature-softmax policy is easy to implement, its probability ratio cannot be computed exactly because the softmax mapping is not injective: $\mathrm{softmax}(\f)=\mathrm{softmax}(\f+c)$ for any $c\in\R$. We therefore consider the logistic-normal distribution as the policy distribution, which admits an analytical probability density \citep{floto2023diffusion}.

Logistic Transform $\mathrm{LT}: \R^{V-1}\to \Delta^{V-1}$ defined by
$$\mathrm{LT}(\y):=\begin{cases}
    \frac{e^{y_i}}{1+\sum_{j=1}^{V-1}e^{y_j}}, & \text{if} \ i \in [V-1],\\
    \frac{1}{1+\sum_{j=1}^{V-1}e^{y_j}}, & \text{if} \ i=V
\end{cases}$$
is a bijection.
At time $t$ given partially masked sequence (state) $\x_t$, we have model logits $\f_\theta(\x_t) \in\R^V$. Vanilla sampling: $\mathrm{softmax}(\f_\theta(\x_t))$. Now consider exploring this logits via Logistic-Normal. Let $p_\theta(\x_t)_j:=f_\theta(\x_t)_j-f_\theta(\x_t)_V$ (subtracting a baseline to reduce one dimension) for all $j\in [V-1]$. Then $\p_\theta(\x_t)\in\R^{V-1}$. Then let the action be
$$\a_t\sim \mathrm{LT}(\cN(\p_\theta(\x_t),\sigma_t^2 I_{V-1})).$$
If $\sigma_t=0$, then $\a_t=\mathrm{softmax}(\f_\theta(\x_t))$. The density of $\a_t$ has a closed form (see, e.g., \citet[Equation (4)]{floto2023diffusion}). Now given stored $\a\sim\pi^{\theta_\mathrm{old}}$, let $\y=\mathrm{LT}^{-1}(\a)\in\R^{V-1}$ and suppose $y_j=p_{\theta_\mathrm{old},j}+\sigma \epsilon_j$ for $j=1,\cdots,V-1$ (In practice, we store $\epsilon_j$ and $p_{\theta_\mathrm{old},j}$ for all $j=1,\cdots,V-1$ during rollout), we have the probability ratio
\begin{align}
\rho^\theta=\frac{\pi^\theta(\a|\x)}{\pi^{\theta_\mathrm{old}}(\a|\x)}={}&\exp\left(
\frac{1}{2\sigma^2}\sum_{j=1}^{V-1}\left((y_j-p_{\theta_\mathrm{old},j})^2-(y_j-p_{\theta,j})^2\right)
\right)\notag\\
={}&\exp\left(
\frac{1}{2\sigma^2}\sum_{j=1}^{V-1}\left(\sigma^2\epsilon_j^2-(p_{\theta_\mathrm{old},j}-p_{\theta,j}+\sigma\epsilon_j)^2\right)
\right),\label{rho:ln}
\end{align}
in which everything is dependent on $t$. 


\begin{remark}\label{rm}
We empirically observe that the three policies described above lead to unstable training. Since these policies are supported over the entire simplex, the probability ratios in \eqref{rho:Dirichlet}, \eqref{rho:temperature}, and \eqref{rho:ln} involve summation over the entire vocabulary, which can result in exploding gradient norms when the vocabulary is large.
Furthermore, the probability ratios may become excessively large during training. In \eqref{rho:Dirichlet}, this is caused by the singularity of $\log \Gamma(\cdot)$, while in \eqref{rho:temperature}, it arises when the old logits are close to zero. Consequently, most probability ratios are clipped into the interval $[1-\epsilon,,1+\epsilon]$ by PPO or GRPO, leading to negligible policy updates and ineffective learning. Although these policies are not suitable for training, we empirically observe that exploring the entire simplex during inference can be beneficial for generating training rollouts.
\end{remark}

\subsection{Policies over the Simplex Vertices}\label{app2-4}
Motivated by these observations stated in Remark~\ref{rm}, in this section we introduce a class of policies that support on the vertices of the simplex. This design avoids the optimization issues associated with full-simplex policies and yields more stable and effective training.
We first recall the d1 loss (Eq (4) in \cite{zhao2025d}) is equivalent to (if we do not consider random masking of prompts)
\begin{equation}\label{d1loss}
\cL_{d1}(\theta):=\E_{q\sim\cD, \{(\x_{g,0},\x_{g,1},\cdots,\x_{g,T})\}_{g=1}^G\sim \pi_{\theta_\mathrm{old}}(\cdot|q)}\left[
\frac{1}{G}\sum_{g=1}^G \frac{1}{L} \sum_{t=1}^{L} \min\left(
\rho_{g,t}^\theta A_{g}, \mathrm{clip}\left(\rho_{g,t}^\theta,1-\epsilon,1+\epsilon\right) A_{g}
\right)
\right],
\end{equation}
where 
\begin{equation}\label{prob_ratio_d1}
\rho_{g,t}^\theta=\frac{p_\theta^{(i_{g,t})}(x_{g,T}^{i_{g,t}}\mid\x_{g,0})}{p_{\theta_\mathrm{old}}^{(i_{g,t})}(x_{g,T}^{i_{g,t}}\mid\x_{g,0})}.
\end{equation}
This loss is autoregressive (AR)-style, where the inner summation in \eqref{d1loss} is computed to average the sequence-length steps. Given this insight, the corresponding d1 sampling procedure here should be understood in the sense that exactly one token is unmasked at each time step; in \eqref{prob_ratio_d1}, $i_{g,t}\in [L]$ is the masked token index/position with the highest confidence in the current block at time step $t$ for the $g$-th denoising trajectory and $x_{g,t}^{i_{g,t}}=\mathtt{m}$ is to be unmasked at time $t$.

However, the actual policy for the d1 sampling is diffusion-style (see Appendix~D in \cite{zhao2025d}), i.e., two masked tokens with the highest confidence score are unmasked at one time step:
\begin{equation}\label{d1sample}
\a_t^{(i)}\sim \pi^{\theta,(i)}(\cdot\mid t,\x_t)\quad \text{then}\quad \a_t^{(i)}=
\begin{cases}
    \e_j \quad \text{w.p.}\ p_\theta^{(i)}(\x_t)_j,\ j\in\cV, &\text{for}\ i\in\cM_t^{\mathrm{Top2Conf}},\\
   \e_{x_t^i}, &\text{for}\ i\notin \cM_t^{\mathrm{Top2Conf}}
\end{cases}\in \Delta^{V-1}.
\end{equation}
Note that the action distribution for a position $i\in \cM_t^{\mathrm{Top2Conf}}$ is supported on the $V$ vertices of the probability simplex $\Delta^{V-1}$. Apparently, there is a gap between the diffusion-style sampling procedure~\eqref{d1sample} and the AR-style loss construction for the d1 loss~\eqref{d1loss}. To be more specific, if we look at the loss \eqref{d1loss} in which the probability ratios are only conditioned on the initial state/prompt $\x_0$ and the inner sum is averaged over the sequence length steps, the sampling procedure is as if
\begin{equation}\label{sample_asif_d1}
\a_t^{(i)}\sim \pi^{\theta,(i)}(\cdot\mid t,\x_t)\quad \text{then}\quad \a_t^{(i)}=
\begin{cases}
\e_j \quad \text{w.p.}\ p_\theta^{(i)}(\x_0)_j,\ j\in\cV, &\text{for}\ i=i_t,\\
   \e_{x_t^i}, &\text{for}\ i \neq i_t
\end{cases}\in \Delta^{V-1},
\end{equation}
which in fact/in reality is not this case as shown in \eqref{d1sample}. The ``as if'' argument is: since the token $x_T^{i_t} \in\cV$ is sampled from $\a_t^{(i_t)}\in\Delta^{V-1}$, by \eqref{sample_asif_d1} we have $\a_t^{(i_t)}=\e_{x_T^{i_t}}$, and thus the policy probability ratio 
$$
\frac{\pi^\theta(\a_t\mid t,\x_t)}{\pi^{\theta_\mathrm{old}}(\a_t\mid t,\x_t)}=\frac{\pi^{\theta,(i_t)}(\a_t^{(i_t)}\mid t,\x_t)}{\pi^{\theta_\mathrm{old},(i_t)}(\a_t^{(i_t)}\mid t,\x_t)}=\frac{p_\theta^{(i_t)}(x_T^{i_t}\mid \x_0)}{p_{\theta_\mathrm{old}}^{(i_t)}(x_T^{i_t}\mid \x_0)},
$$
which coincides with the probability ratio \eqref{prob_ratio_d1} of d1.
Motivated by our continuous-time RL perspective, we bridge this gap by revising the d1 loss into our loss \eqref{ourloss}, which is consistent with the underlying continuous-time RL dynamics and the inference procedure. 

However, evaluating \eqref{ourloss} exactly requires one forward
pass on each reconstructed state $\x_{g,t}$, for every step $t=1,\dots,T$ with
gradients enabled. With $T$ on the order of $128$--$512$ this is the dominant cost
in both memory and wall-clock time: the autograd graph for every step is held until the backward pass. To address this, one can consider gradient checkpointing across the step loop by recomputing each step's forward during backward instead of storing all activations. This keeps the math identical and brings peak memory back to roughly one forward graph at a time, at the cost of one extra forward pass per step. The extra recomputation makes the policy pass slower, with roughly $2\times$ the forwards for that pass. The other practical option is to subsample the trajectory described in Section~\ref{sec:posttraining}, compute the loss on a random subset of $N$ denoising steps each update, trading exactness for memory. Trajectory subsampling reduces the number of evaluated steps per optimizer update
from $T$ to a small $N\ll T$, giving an approximately $T/N$ reduction in the
number of policy forward/backward passes, at the cost of a stochastic but unbiased estimate (Proposition~\ref{unbias}) of the per-rollout mean. In Figure~\ref{fig:ablation_N_gsm8k}, we
show the accuracy trade-off of CTRL with different values of $N$ in the GSM8K benchmark. We found that $N=8$ is a favorable balance between the accuracy and training efficiency, which reduces the convergence time without compromising the final reward. 

\begin{figure}[!htbp]
\centering
\includegraphics[width=0.6\linewidth]{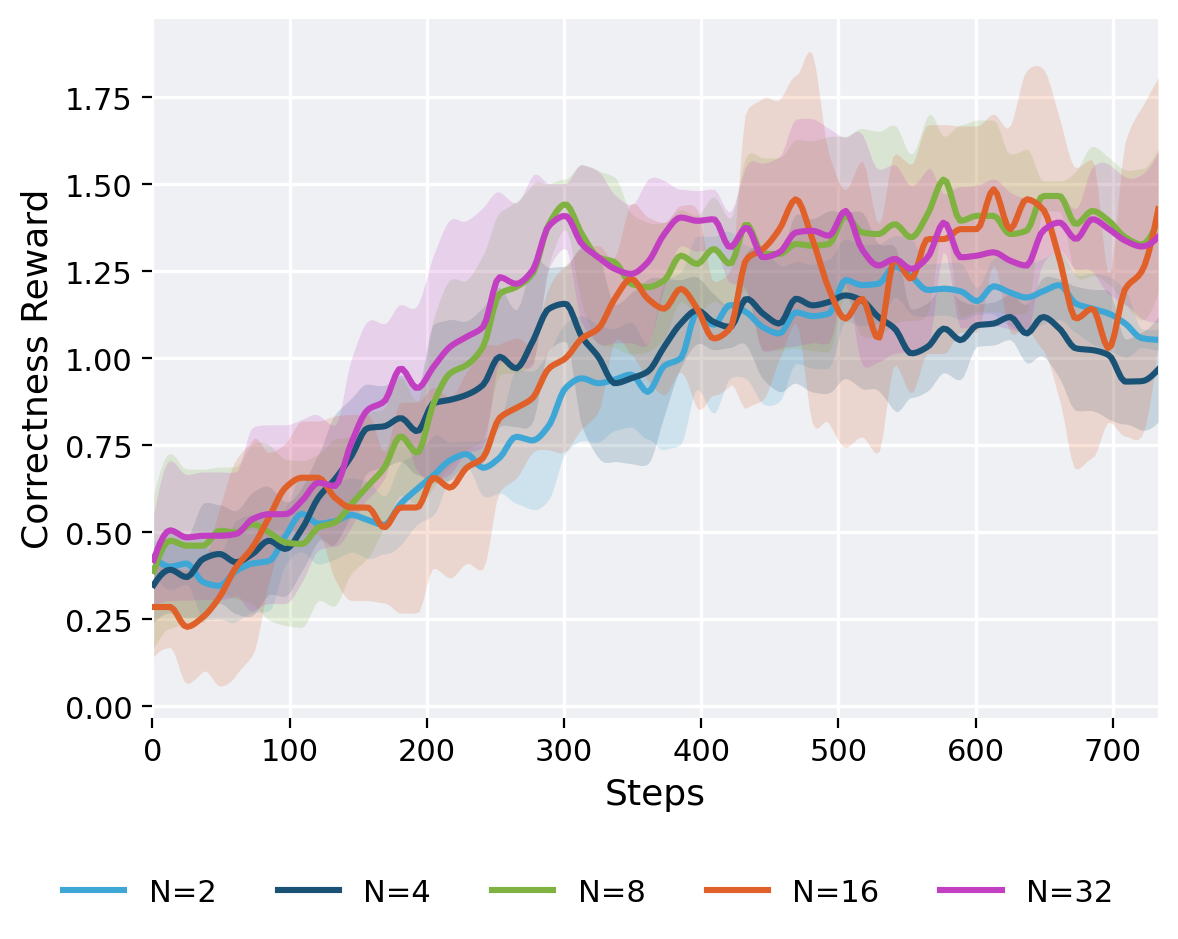}
\caption{Ablation on $N$ on GSM8K.}\label{fig:ablation_N_gsm8k}
\end{figure}

\begin{proposition}\label{unbias}
Write
$\ell_{t}(\theta):=\min(\rho_{t}A_{t},\mathrm{clip}(\rho_{t},1-\epsilon,1+\epsilon)A_{t})$
for the per-step GRPO surrogate and $\bar{\ell}_t(\theta):=\sum_{t=1}^T \ell_t(\theta)/T$. For a subset $\cT$ drawn uniformly without replacement from
$[T]$ with $|\cT|=N$,
\[
  \E_{\cT}\big[\widehat{\ell}_{\cT}(\theta)\big]
  \;=\;
  \bar\ell(\theta).
\]
\end{proposition}

\textit{Proof of Proposition~\ref{unbias}.}
For uniform sampling without replacement, each index $t$ has marginal inclusion
probability $\P(t\in\cT)=N/T$. Hence, 
\[
  \E_{\cT}\big[\widehat{\ell}_{\cT}(\theta)\big]
  = \frac{1}{N}\sum_{t=1}^{T}\P(t\in\cT)\,\ell_t(\theta)
  = \frac{1}{N}\sum_{t=1}^{T}\frac{N}{T}\,\ell_t(\theta)
  = \frac{1}{T}\sum_{t=1}^{T}\ell_t(\theta)
  = \bar\ell(\theta).
\]
$\hfill\square$

\section{PPO Implementation Insights for DLLMs}\label{app3}
In this section, we present an alternative approach to value function approximation for PPO that avoids training a value network by leveraging sampled rollouts and a terminal reward function, inspired by DISPO \citep{oba2026diffusion}. In PPO, computing the intermediate advantage requires estimating the value function: for $\a_t\sim \pi^\theta(\cdot\mid t,\x_t)$, we can write 
$$
q(t,\x_t,\a_t;\pi^\theta(\cdot\mid t,\x_t))=\frac{1}{1-t}\sum_{i\in\cM_t^{\mathrm{TopkConf}}}\sum_{j\in\cV} (\a_t^{(i)})_j\left(V^{\theta}(t,\x_t^{\backslash i}\odot j)-V^{\theta}(t,\x_t)-\beta\log\frac{p_\theta^{(i)}(\x_t)_j}{p_{\theta_\mathrm{pre}}^{(i)}(\x_t)_j}\right).
$$
One may drop the scalar factor $\frac{1}{1-t}$ and set $\beta=0$ first, then we have $q_t:=q(t,\x_t,\a_t;\pi^\theta(\cdot\mid t,\x_t))\approx\sum_{i\in\cM_t^{\mathrm{TopkConf}}}\sum_{j\in\cV} (\a_t^{(i)})_j\left(V^{\theta}(t,\x_t^{\backslash i}\odot j)-V^{\theta}(t,\x_t)\right)$. Noting that $\a_t^{(i)}$ is a probability distribution over $\cV$, we can write
\begin{align*}
q_t\approx{}&\sum_{i\in\cM_t^{\mathrm{TopkConf}}} \E_{j \sim \a_t^{(i)}}\left(V^{\theta}(t,\x_t^{\backslash i}\odot j)-V^{\theta}(t,\x_t)\right)\\
\approx{}& \sum_{i\in \cM_t^{\mathrm{TopkConf}}} \left(V^{\theta}(t,\x_t^{\backslash i}\odot x_{t+1}^i)-V^{\theta}(t,\x_t)\right)\\
\approx{}& V^\theta(t,\x_{t+1})-V^\theta(t,\x_t),
\end{align*}
where the last approximation is because we simultaneously unmask the $k$ tokens $x_{t+1}^i\ (\text{for} \ i \in \cM_t^{\mathrm{TopkConf}})$ at time step $t$. To estimate the value functions, in the rollouts, for each time step $t$, we resample $Z$ outputs $\{\o_{t,z}\}_{z=1}^Z$ from $\x_t$ by one-step generation using the current logits $\f_\theta(\x_t)$, and then estimate $V^\theta(t,\x_t)$ via the Monte Carlo average
$\frac{1}{Z}\sum_{z=1}^Z \mathrm{RM}(\o_{t,z})$. Hence, we have the estimator
$$
q_t\approx \frac{1}{Z} \sum_{z=1}^Z \left(\mathrm{RM}(\o_{t+\Delta t,z})-\mathrm{RM}(\o_{t,z})\right).
$$

In a rollout generated by $\pi^{\theta_{\mathrm{old}}}$, given $\x_t$ at time $t$, we resample $Z$ outputs $\{\o_{t,z}\}_{z=1}^Z$ where each $\o_{t,z}$ fills all the masked positions of $\x_t$ via one-step generation according to $\mathrm{softmax}(\f_{\theta_{\mathrm{old}}}^{(i)}(\x_t))$ ($i\in [L]:x_t^i=\mathtt{m}$), and compute $V_t:=\frac{1}{Z}\sum_{z=1}^Z \mathrm{RM}(\o_{t,z})$. We store only $V_t$ and $V_{t+1}$ in order to compute $q_t=V_{t+1}-V_t$. After all $\{q_t\}_{t\in [T]}$ have been computed and stored, the intermediate values $\{V_t\}_{t\in [T]}$ are discarded, thereby reducing the memory footprint. However, we found that this estimator for $q$ leads to unstable training and exhibits high variance, possibly because its estimates are coarse at small $t$ during the early denoising steps due to one-step generation. We leave further stabilization techniques and engineering improvements for PPO to future work.

\section{Experiment Details}
\subsection{Synthetic Example}\label{app4-1}

DAM gives the entropy-regularized optimum $p^*(\mathbf{X}_1)\propto p^*(\mathbf{X}_1)e^{h(\mathbf{X}_1)}$:
\begin{equation}
  p^*(i,j)\propto p^\mathrm{base}(i,j)\,\exp\!\big(h(i,j)/\beta\big)
  =\frac{\exp\!\big(h(i,j)/\beta\big)}{\sum_{a,b}\exp\!\big(h(a,b)/\beta\big)},
  \label{eq:target}
\end{equation}
since $p^\mathrm{base}$ is uniform.

\textbf{Reward Functions.} Following \cite{so2026discrete}, the (terminal) reward $h(\cdot)$ is assigned via upweighting the diagonal blocks by $4.6$, superdiagonal and subdiagonal blocks by $4.0$, and other off-diagonal blocks by $3.4$. 

\textbf{Implementation.} All experiments are conducted on CPU. The hyperparameters for the three methods are summarized in Table~\ref{tab:checkerboard_compare}.

\begin{table}[h]
\centering
\small
\begin{tabular}{llll}
\toprule
\textbf{Quantity} & \textbf{PPO} & \textbf{DAM} & \textbf{D1} \\
\midrule
Objective & value matching & adjoint matching & reward max.\ (GRPO) \\
Critic & exact tabular & none (adjoint) & none (critic-free) \\
Estimator & exact expectation & MC adjoint, $K{=}12$ & one-step log-prob \\
Group size $G$ & --- & --- & $64$ \\
Inner updates $\mu$ & $4$ & $1$ & $4$ \\
Clip $\epsilon$ & $0.2$ & --- & $0.2$ \\
KL weight $\beta$ & $6.0$ & $6.0$ & $6.0$ \\
Learning rate & $0.3$ & $0.3$ & $0.3$ \\
\bottomrule
\end{tabular}
\caption{Hyperparameters. The problem definition is identical across
methods, so performance differences are attributable to the algorithm.}\label{tab:checkerboard_compare}
\end{table}

\paragraph{Evaluation.}
After each iteration we compute the exact terminal grid $p_\theta(i,j)$ and report
$\mathrm{KL}(p^*\,\Vert \,p_\theta)$ against \eqref{eq:target}, the average reward
$\E_{p_\theta}[h]$, and the aggregated mass of each of the $25$ blocks. Table~\ref{tab:checkerboardnumber} reports the quantitative convergence results in terms of KL divergence and average reward, corresponding to the last two subfigures in Figure~\ref{fig:compare_checkerboard}.

\begin{table}[h]
\centering
\begin{tabular}{lcc}
\toprule
\textbf{Method} & $\mathrm{KL}(p^*\,\Vert\,p_\theta)$  & avg.\ reward \\
\midrule
PPO            & $0.00044$ & $2.793$ \\
DAM                   & $0.004$  & $2.560$ \\
D1 & $0.0769$\, & $2.282$ \\
\midrule
optimal $p^*$      & $0$ & $2.850$ \\
\bottomrule
\end{tabular}
\caption{Final performance on the checkerboard task ($400$ iterations).}
\label{tab:checkerboardnumber}
\end{table}

\subsection{Mathematical Reasoning}\label{app4-2}
\textbf{Inference.} Following d1 \citep{zhao2025d}, our training rollouts are generated by Top2 confidence (i.e., $k=2$ in the main text) and random sampling using the semi-autoregressive decoding strategy. Specifically, we have $T=L/2$ and the $2$ tokens with the highest confidence within the current block are unmasked, with the unmasked tokens sampled from $\p_\theta(\x_t)$, at each time step. The block size is set to $32$, and once
all the tokens in the current block are unmasked, we move to the next block of $32$ tokens. To encourage exploration, we adopt a dimension-aware temperature-softmax strategy, where the temperature is sampled from a time-independent exponential distribution, described in Appendix~\ref{app2-2}, to encourage exploration:
$$
\p_\theta^{(i)}(\x_t)=\mathrm{softmax}((f_\theta^{(i)}(\x_t)/\tau_j)_{j\in \cV}),\quad i\in \cM_t^\mathrm{Top2Conf},
$$
where $\tau_j \overset{\text{i.i.d.}}{\sim} \mathrm{Exp}(\lambda),\ j\in\cV$. In Figure~\ref{fig:ablation_lambda_gsm8k}, we show the ablation results with different choices of $\lambda$ on GSM8K with a sequence length of 128. We can see that $\lambda=2.0$ yields a slightly better result, which we take in our experiments. When $\lambda$ is either too large or too small, the exploration becomes excessive, causing the model logits to be overly perturbed and consequently degrading generation quality.

The terminal reward functions $\mathrm{TRF}(\cdot)$ for the four tasks are verifiable reward functions taken the same as d1 (see Appendix D.1.1 in \cite{zhao2025d}). Our intermediate verifiable reward functions $\mathrm{IRF}(\cdot)$ are specified as follows:
\begin{itemize}
    \item \textbf{GSM8K.} The intermediate reward consists of three non-positive components: 

    \begin{enumerate}
    \item an \textit{answer-region legality penalty}, applying $-1.0$ if the \verb|<answer>|$\cdots$\verb|</answer>| region is fully visible but contains any visible character that does not belong to the integer category (letters, currency symbols, commas, decimal points, or digit runs split by visible content); 
    
    \item a \textit{trailing-content penalty} of $-0.001$ per visible character after \verb|</answer>|. The maximal amount of penalty from this component is capped at $-0.5$ ;
    
    \item a \textit{tag-structure penalty} of $-0.125$ per duplicated fully visible tag and $-0.5$ for an impossible tag order (e.g. \verb|</answer>|$\cdots$\verb|<answer>|). Masked positions are never penalized.
    \end{enumerate}
    
    \item \textbf{MATH500.} The intermediate reward consists of three non-positive components:
    \begin{enumerate}
    \item an \textit{answer-region legality penalty}, adapted to the \verb|\boxed{}| format of the terminal reward: applying $-1.0$ if the \verb|<answer>|$\cdots$\verb|</answer>| region is fully visible but contains no \verb|\boxed{|$\cdots$\verb|}| expression;

    \item a \textit{trailing-content penalty} of $-0.001$ per visible character after \verb|</answer>|. The maximal amount of penalty from this category is capped at $-0.5$;

    \item a \textit{tag-structure penalty} of $-0.125$ per duplicated fully visible tag and $-0.5$ for an impossible tag order. Masked positions are never penalized.
    \end{enumerate}
    
    \item \textbf{Countdown.} The sequence is split at the first fully visible \verb|<answer>| tag, and the intermediate reward consists of two components, with target correctness left to $\mathrm{TRF}(\cdot)$:
    \begin{enumerate}
    \item an \textit{equation-shaping award} on the reasoning region: each unique, fully visible arithmetic equation is verified with exact rational arithmetic and awarded $+0.05$ for correctness and $-0.5$ otherwise. Partially masked equations are skipped;

    \item an \textit{answer-region legality penalty}, applying a flat $-1.0$ if any illegal visible character appears (only digits, $+$, $-$, $*$, $/$, parentheses, and whitespace are permitted) or if a fully determined digit run does not match an available number.
    \end{enumerate}

    \item \textbf{Sudoku.} The intermediate penalty is evaluated in the following steps:
    \begin{enumerate}
    \item Pad or truncate the extracted answer to a 16-character grid as in the terminal validator, and only the cells from positions empty in the original puzzle are inspected;

    \item Check each inspected cell for \textit{value legality}, only a value belongs to $\{1,2,3,4\}$ is legal;

    \item Compute intermediate penalty according to \[r_{\text{intermediate}}=\frac{\#\{\text{invalid inspected grids}\}}{\#\{\text{all inspected grids}\}}.\]
    \end{enumerate}
\end{itemize}
All four $\mathrm{IRF}$'s are mask-tolerant: they score partially denoised sequences by grading only visible content. The intermediate scores are scaled by a tunable weight hyperparameter $\alpha_{\mathrm{intermediate}}$ ($0.05$ by default).

\textbf{Implementation.} We build our implementation on top of the same codebase of d1\footnote{https://github.com/dllm-reasoning/d1}. Our CTRL uses Low-Rank Adaptation (LoRA) \citep{hu2022lora} with a rank of $r=128$ and scaling factor $\alpha=64$, and AdamW \citep{loshchilov2018decoupled} parameters $(\beta_1,\beta_2)=(0.9,0.99)$ with a weight decay of $0.1$ and learning rate $3\times 10^{-6}$. All our training is conducted on 7 NVIDIA L40S GPUs, batch size of $6$ per GPU and use gradient accumulation steps of $2$. For GRPO hyperparameters, the group size $G=6$, the number of inner gradient update is $6$, and the clip parameter $\epsilon=0.5$. The same configuration is also used for the three baseline training, and we set $N=16$ in d2 \citep{wang2026d2} as the authors suggest.

\begin{figure}[!htbp]
\centering

\begin{minipage}{0.48\linewidth}
    \centering
    \includegraphics[width=\linewidth]{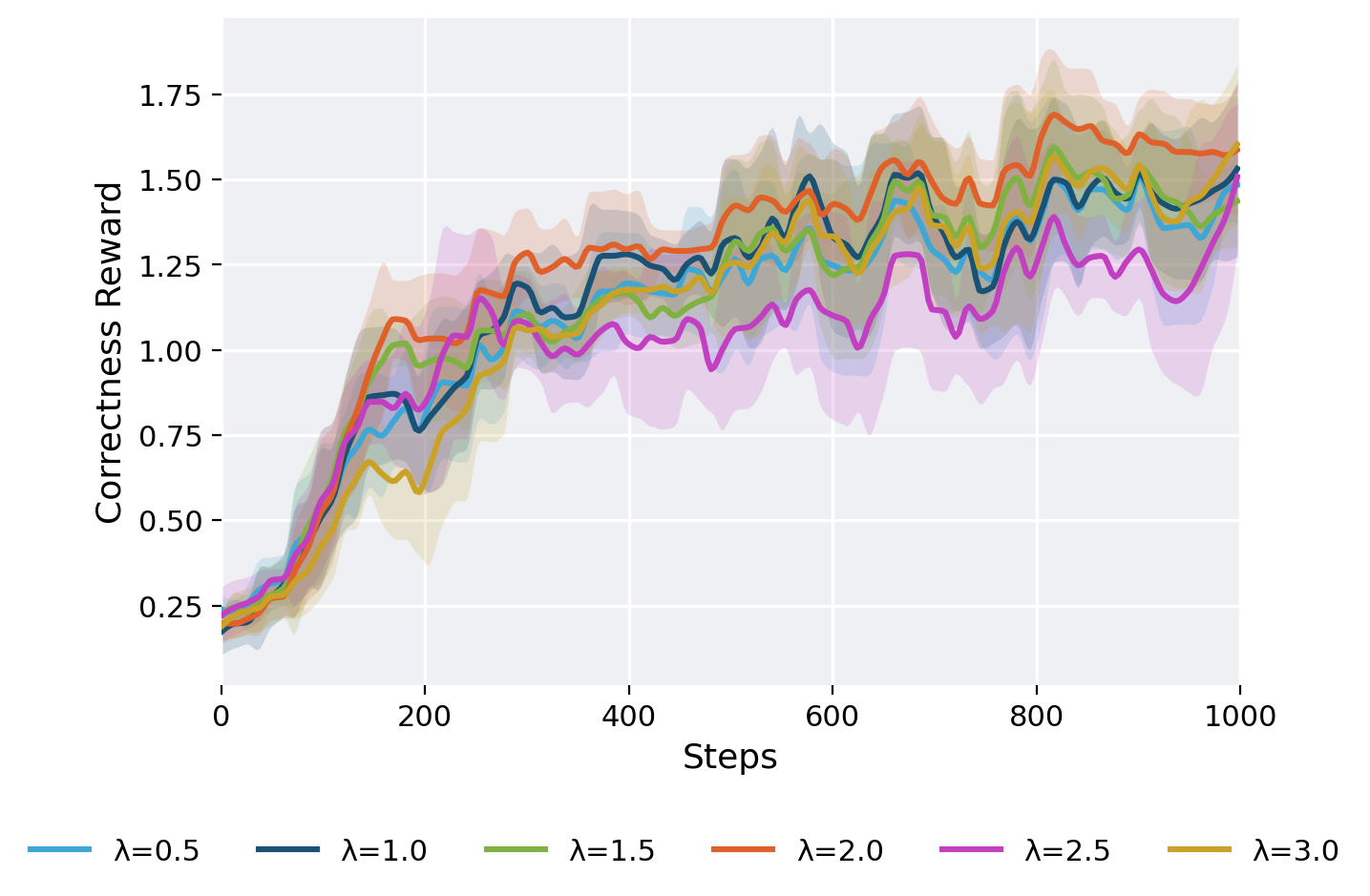}
    \caption{Ablation on $\lambda$ on GSM8K.}
    \label{fig:ablation_lambda_gsm8k}
\end{minipage}
\hfill
\begin{minipage}{0.44\linewidth}
    \centering
    \includegraphics[width=\linewidth]{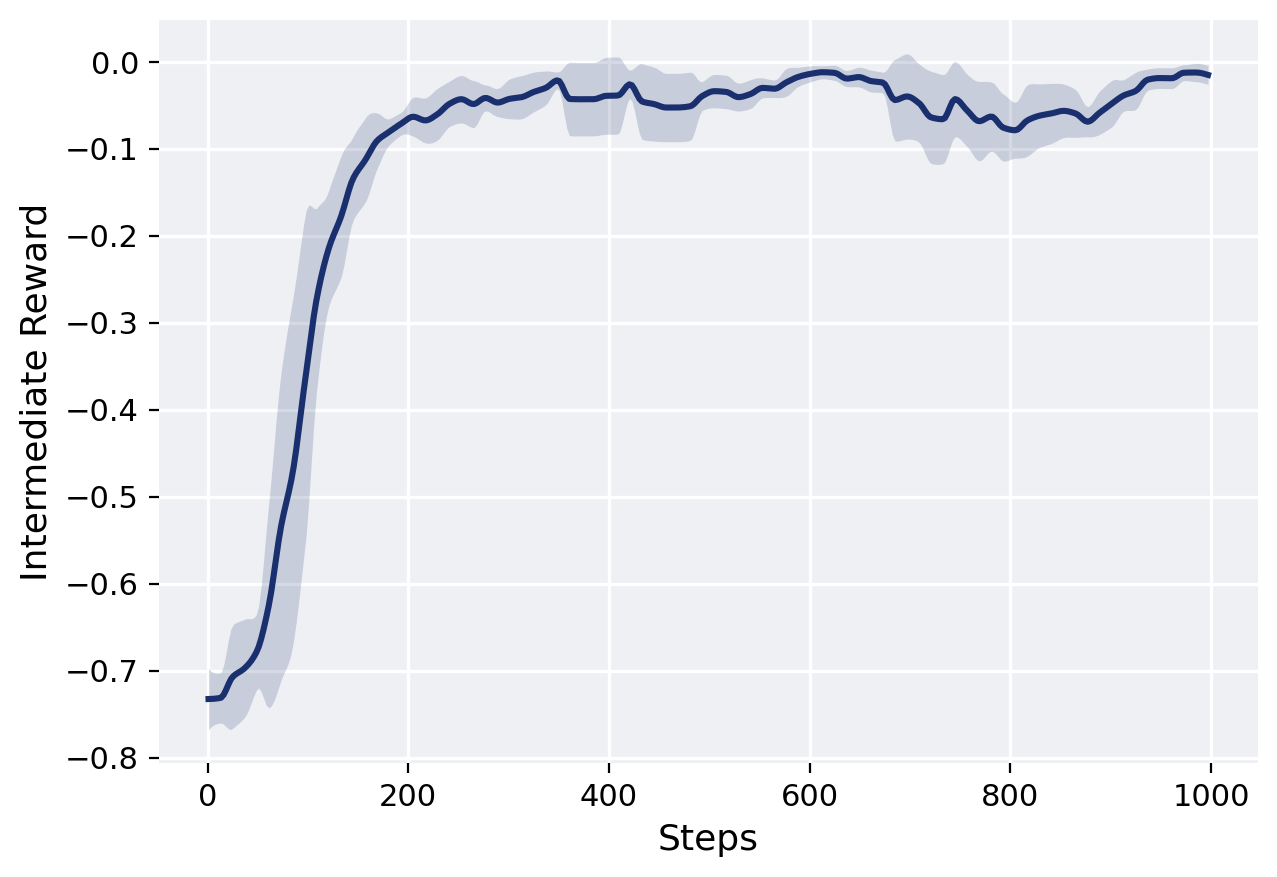}
    \caption{Convergence of intermediate reward of Sudoku.}
    \label{fig:sudoku_interreward}
\end{minipage}

\end{figure}

\end{appendix}